\documentclass{article}

\usepackage{microtype}
\usepackage{graphicx}
\usepackage{subfigure}
\usepackage{booktabs} 

\usepackage{hyperref}

\usepackage[accepted]{icml2024}

\usepackage{amsmath}
\usepackage{amssymb}
\usepackage{mathtools}
\usepackage{amsthm}

\usepackage[capitalize,noabbrev]{cleveref}

\theoremstyle{plain}
\newtheorem{theorem}{Theorem}[section]
\newtheorem{proposition}[theorem]{Proposition}
\newtheorem{lemma}[theorem]{Lemma}

\theoremstyle{definition}
\newtheorem{definition}[theorem]{Definition}

\theoremstyle{remark}

\usepackage[textsize=tiny]{todonotes}

\icmltitlerunning{Adversarial Constrained Policy Optimization}

\begin{document}

\twocolumn[
\icmltitle{Adversarial Constrained Policy Optimization: Improving Constrained Reinforcement Learning by Adapting Budgets}



\icmlsetsymbol{equal}{*}

\begin{icmlauthorlist}
\icmlauthor{Jianming Ma}{automation}
\icmlauthor{Jingtian Ji}{paris}
\icmlauthor{Yue Gao}{ai}

\end{icmlauthorlist}

\icmlaffiliation{automation}{Department of Automation, Shanghai Jiao Tong University, Shanghai, China}
\icmlaffiliation{paris}{ParisTech Elite Institute of Technology, Shanghai Jiao Tong University, Shanghai, China}
\icmlaffiliation{ai}{MoE Key Lab of Artificial Intelligence and AI Institute, Shanghai Jiao Tong University, Shanghai, China}

\icmlcorrespondingauthor{Yue Gao}{yuegao@sjtu.edu.cn}

\icmlkeywords{Constrained Reinforcement Learning}

\vskip 0.3in
]



\printAffiliationsAndNotice{}  

\begin{abstract}
Constrained reinforcement learning has achieved promising progress in safety-critical fields where both rewards and constraints are considered. 
However, constrained reinforcement learning methods face challenges in striking the right balance between task performance and constraint satisfaction and it is prone for them to get stuck in over-conservative or constraint violating local minima. 
In this paper, we propose Adversarial Constrained Policy Optimization (ACPO), which enables simultaneous optimization of reward and the adaptation of cost budgets during training. Our approach divides original constrained problem into two adversarial stages that are solved alternately, and the policy update performance of our algorithm can be theoretically guaranteed.
We validate our method through experiments conducted on Safety Gymnasium and quadruped locomotion tasks. Results demonstrate that our algorithm achieves better performances compared to commonly used baselines.

\end{abstract}

\section{Introduction}
\label{section:introduction}

Recent advances have demonstrated the potential of reinforcement learning (RL) in addressing complex decision-making and control problems across various domains. \cite{silver2017mastering, deepmind2019mastering, afsar2022reinforcement, zhuang2023robot, kiran2021deep} The RL framework is rooted on the reward hypothesis \cite{silver2021reward}, which suggests that all tasks can be modeled as maximizing a single reward function.
However, the practical challenge lies in selecting an appropriate form of the reward function that can elicit the desired behavior from the agent. This challenge becomes particularly pronounced in safety-critical tasks, where constraints on agent behavior often exist. Merely intuitively adjusting the reward function to balance constraint satisfaction and task performance poses a formidable challenge.

Constrained RL presents a promising avenue for generating desired behaviors that satisfy constraints. In the Constrained RL framework, the agent aims to maximize the reward function while simultaneously satisfying a set of constraints defined by additional cost functions and cost budgets. Unlike reward functions, cost functions and budgets can be designed in a manner that holds clear physical meaning \cite{roy2021direct, kim2023not, gangapurwala2020guided}, considerably facilitating parameter tuning.

Nonetheless, Constrained RL encounters a trade-off between task performance and constraint satisfaction, which amplifies the difficulty of training. Most Constrained RL algorithms \cite{cpo, tessler2018reward, ipo, ray2019benchmarking, crpo, cup} incorporate constraints into the policy optimization process by updating the local policy using policy gradient method and constrained optimization techniques. Throughout the training process, the cost budgets, which quantitatively represent the strength of constraints, remain unchanged. One challenge is that the policy tends to get stuck at an over-conservative sub-optimal solution, where constraints are satisfied but the reward is low.

One potential approach to enhance sub-optimal solutions in Constrained RL involves adjusting the cost budget during training to encourage policy exploration. Prior studies \cite{sootla2022enhancing, yangself} have found that policy performance can be improved by intuitively changing the cost budget during training, like the way of curriculum learning. However, this strategy heavily relies on expert experience to pre-define a desired budget trajectory before training. Furthermore, this method lacks performance guarantees when the budget undergoes changes.

In this study, we propose a novel policy optimization strategy for Constrained RL that enables adaptive optimization of both the reward return and cost budget. Specifically, we decompose the original problem into two adversarial stages: one stage focuses on maximizing the reward under the current cost budget, while the other stage aims to minimize the cost while ensuring the reward remains above the current reward budget. Throughout the training process, these two stages are alternately solved, and the cost and reward budgets are updated based on the solutions. The algorithm is referred to as Adversarial Constrained Policy Optimization due to the adversarial nature of the two stages.
By adversarially learning the cost budget and reward return, our algorithm is more likely to escape local minima and achieve improved rewards while still satisfying constraints. Furthermore, we provide theoretical lower bounds on the performance of our algorithm's updates, offering insights into its efficacy.

Our contributions are summarized as follows:
\begin{itemize}
    \item We propose Adversarial Constrained Policy Optimization (ACPO), an constrained RL framework that can simultaneously optimizing rewards and adapting cost budgets during training to improve final performances.
    \item We use two adversarial stages that solved alternately to balance the trade-off between reward performance and constraint satisfaction, and theoretically analyze the lower bound of policy update of our algorithm.
    \item We conduct experiments on Safety Gymnasium and quadruped locomotion tasks. Results demonstrate ACPO can outperform other baselines with higher reward under the same cost budget.  
\end{itemize}

\section{Related Works}
\label{section:related_works}

Constrained reinforcement learning is a class of generalized RL methods that aim to maximize reward while adhering to specific constraints. In this letter, our focus does not encompass safe exploration, i.e., the policy may violate constraints during the training process, but it needs to satisfy constraints when it eventually converges. Depending on the approach used to tackle the optimization problem, we can broadly classify the methods into the following categories: \par
{\bf Lagrange methods}: In these methods, the Lagrange relaxation technique is used to convert original problem into an unconstrained problem. Extensive canonical algorithms are proposed based on the Lagrange approach, such as \cite{chow2018risk, tessler2018reward, liang2018accelerated, paternain2019constrained, stooke2020responsive, ding2021provably, chen2021primal}. These methods introduce Lagrange multipliers into the problem, and the initialization and learning rate of multipliers are sensitive to overall performances. \par
{\bf Primal methods}: These methods \cite{dalal2018safe, chow2018lyapunov, cpo, pcpo, crpo, ipo, p3o} solve the original constrained problem without optimizing Lagrange multipliers. CPO \cite{cpo} replaces the objective and constraint with surrogate functions and can ensure monotonic performance improvement and constraint satisfaction. It solves the primal problem within a trust region using appropriate approximations. PCPO \cite{pcpo} splits the primal problem into two steps by first updating policy using regular TRPO \cite{trpo}, and then projecting policy into feasible region. CUP \cite{cup} provides theoretical analysis of surrogate function based methods combining with generalized advantage estimator (GAE) \cite{gae}, and gives tighter surrogate bounds. CRPO \cite{crpo} uses an alternate update strategy to solve the primal problem and gives theoretical convergence guarantees. Some  research use penalty functions to present constraints, e.g. IPO \cite{ipo} and P3O \cite{p3o}. These methods demonstrate that 
the solution of unconstrained problem using barrier function can converge to the optimal solution of the primal problem.

\section{Preliminaries}
\label{section:preliminaries}

Constrained reinforcement learning is often formulated as a Constrained Markov Decision Process (CMDP) \cite{cmdp} that is a tuple $\mathcal{M}=(\mathcal{S},\mathcal{A},P,R,\mathcal{C},\rho_0,\gamma)$, where $\mathcal{S}$ is the state space, $\mathcal{A}$ is the action space, $P:\mathcal{S}\times\mathcal{A}\times\mathcal{S}\rightarrow [0,1]$ is the transition model, $R:\mathcal{S}\times\mathcal{A}\rightarrow \mathbb{R}$ is the reward function, $\mathcal{C}$ is the set of cost functions $C_i:\mathcal{S}\times\mathcal{A}\rightarrow \mathbb{R},i=1,\cdots,m$, $\rho_0:\mathcal{S}\rightarrow [0,1]$ is the initial state distribution, $\gamma\in (0,1)$ is the discounted factor. Let $\pi:\mathcal{S}\rightarrow P(\mathcal{A})$ denote a policy mapping the state space to the space of probability distributions over actions, and $\pi(a|s)$ denotes the probability of selecting action $a$ in state $s$. Let $\tau=(s_0,a_0,\cdots,)$ denote a sample trajectory. The value function and action-value function are defined as 
\begin{gather}
    V_R^\pi (s)=\mathbb{E}_{\tau\sim\pi}\Big[\sum_{t=0}^\infty \gamma^t R(s_t,a_t)|s_0=s \Big] \\
    Q_R^\pi(s,a)=\mathbb{E}_{\tau\sim\pi}\Big[\sum_{t=0}^\infty \gamma^t R(s_t,a_t)|s_0=s,a_0=a \Big]
\end{gather}
And the advantage function is defined as $A_R^\pi(s,a)=Q^\pi(s,a)-V^\pi(s,a)$. 
Similarly, we can define the cost value functions, cost action-value functions and cost advantage functions as $V_{C_i}^\pi, Q_{C_i}^\pi, A_{C_i}^\pi, i=1,\cdots,m$ by replacing the reward $R$ with $C_i$. Define the reward return as $J_R(\pi)=\mathbb{E}_{s\sim\rho_0}[V_R^\pi(s)]$ and the $C_i-return$ as $J_{C_i}(\pi)=\mathbb{E}_{s\sim\rho_0}[V_{C_i}^\pi(s)]$, we aim to find a feasible policy that maximum the reward return: $\pi^*=\arg\max_{\pi\in\Pi_{\mathcal{C}}} J_R(\pi)$, where $\Pi_{\mathcal{C}}=\{\pi\in\Pi: J_{C_i}(\pi)\leq d_i,i=1,\cdots,m\}$ is the set of feasible policies. \par

Let $\mathbf{P}_\pi \in \mathbb{R}^{|\mathcal{S}|\times|\mathcal{S}|}$ be a state transition probability matrix, and its $(s,s')$-th component is: $\mathbf{P}_\pi(s,s')=\sum_{a\in\mathcal{A}}P(s'|s,a)\pi(a|s)$, which denotes one-step state transition probability from $s$ to $s'$ by executing $\pi$.
We use $P_\pi(s_t=s'|s)=\mathbf{P}_\pi^t(s,s')$ to denote the probability of visiting the state $s'$ after $t$ steps from the state $s$ by executing $\pi$. 
Then we define the discounted state visitation distribution on initial distribution $\rho_0$ as
\begin{equation}
    d_\pi(s)=(1-\gamma)\mathbb{E}_{s_0\sim \rho_0}\Big[\sum_{t=0}^\infty \gamma^t P_\pi(s_t=s|s_0) \Big]
\end{equation}

\subsection{Constrained Policy Update} 
\label{section:constrained_policy_update}

The Constrained RL framework aims to solve the following problem:
\begin{equation}
\begin{gathered}
    \pi^* = \arg\max_{\pi \in \Pi_\theta}\; J_R(\pi)  \\
    \mathrm{s.t.}\;\; J_C(\pi) \leq d
\end{gathered}    
\end{equation}
where $d \geq 0$ is the desired cost budget, $\Pi_\theta \subseteq \Pi$ denotes a set of parametrized policies with parameters $\theta$. It is intractable to solve above problem directly due to the curse of dimensionality, so local policy update is commonly used to solve the problem iteratively:
\begin{equation}
\begin{gathered}
    \pi_{k+1} = \arg\max_{\pi \in \Pi_\theta}\; J_R(\pi) \\
    \mathrm{s.t.}\;\; J_C(\pi) \leq d,\quad D(\pi||\pi_k) \leq \delta
\end{gathered}    
\end{equation}
where $D(\cdot||\cdot)$ denotes a distance measure between policies, $\delta$ is a small positive scalar. Specifically, the problem can be rewritten as the follow using surrogate functions, according to Lemma \ref{lemma:CPO_bound}:
\begin{equation}
\label{eq:optimization_formulation}
\begin{gathered}
    \pi_{k+1} = \arg\max_{\pi\in\Pi_{\theta}} \mathop{\mathbb{E}}\limits_{s\sim d_{\pi_k} \atop a\sim \pi}[A^{\pi_k}_R(s,a)]    \\
    \mathrm{s.t.}\;\; J_C(\pi_k) + \frac{1}{1-\gamma}\mathop{\mathbb{E}}\limits_{s\sim d_{\pi_k} \atop a\sim \pi} [A^{\pi_k}_C(s,a)] \leq d   \\
    \bar{D}_{KL}(\pi||\pi_k) \leq \delta
\end{gathered}
\end{equation}

where $\bar{D}_{KL}(\pi||\pi_k)=\mathbb{E}_{s\sim d_{\pi_k}}[D_{KL}(\pi||\pi_k)[s]]$, $D_{KL}(\cdot||\cdot)$ denotes the KL divergence between distributions of two policies.

\begin{lemma}[policy performance bound \cite{cpo}]
\label{lemma:CPO_bound}
For any policies $\pi$ and $\pi'$, the following bounds hold:
\begin{gather}
J(\pi') - J(\pi) \geq \frac{1}{1-\gamma} \mathop{\mathbb{E}}\limits_{s\sim d_{\pi} \atop a\sim\pi'}[A^{\pi}(s,a)]  \nonumber \\
 - \frac{\sqrt{2}\gamma \epsilon^{\pi'}}{(1-\gamma)^2} \sqrt{\mathbb{E}_{s\sim d_\pi}[D_{KL}(\pi'||\pi)[s]]}  \\
J(\pi') - J(\pi) \leq \frac{1}{1-\gamma} \mathop{\mathbb{E}}\limits_{s\sim d_{\pi} \atop a\sim\pi'}[A^{\pi}(s,a)]  \nonumber \\
 + \frac{\sqrt{2}\gamma \epsilon^{\pi'}}{(1-\gamma)^2} \sqrt{\mathbb{E}_{s\sim d_\pi}[D_{KL}(\pi'||\pi)[s]]}   
\end{gather}
where $\epsilon^{\pi'}=\max_s|\mathbb{E}_{a\sim\pi'}[A^{\pi}(s,a)]|$.
\end{lemma}

There exists many approaches to solve above optimization problem, e.g. CPO \cite{cpo}, PCPO \cite{pcpo} approximate the objective and constraint functions using Taylor expansion and solve the quadratic optimization problem within a trust region; IPO \cite{ipo}, P3O \cite{p3o} handle the KL divergence constraint through PPO \cite{ppo} method, and use penalty function to transform the cost constraint into additional objective term. 

In the above framework, the cost budget $d$ is set to the desired value and unchanged during the whole process of training. That is helpful for the agent to achieve conservative exploration during the training process which is crucial for some safety-critical scenarios. However, in this paper, we focus solely on ensuring constraint satisfaction at the end of policy convergence. During the training process, we allow for some acceptable constraint violations to facilitate exploration. 
Furthermore, we posit that by appropriately adjusting the strength of constraints, such as relaxing constraints during the initial stages of training, the policy can ultimately achieve improved performance. Previous experiments conducted in \cite{yangself} demonstrate that adjusting the cost budget using curriculum learning results in higher rewards compared to using a fixed cost budget. Nevertheless, the adjustment of the cost budget heavily relies on prior knowledge of the environment and expert experience. Hence, finding a generative and adaptive approach to adjust the cost budget, effectively enhancing the final performance, remains an open problem.

\section{Adversarial Constrained Policy Optimization}
\label{section:methodology}

\subsection{Algorithm Overview}
For a better trade off between the reward performance and constraint satisfaction, we develop ACPO, an iterative method that can optimize reward and adjust cost budget simultaneously. \par

We consider the following problem: At the beginning of the training process, there are a large initial cost budget $d_0 \geq d_{desired}$ and a initial reward return $g_0$. How to adaptively adjust the cost budget and optimize the reward, in order to finally satisfy desired constraints and achieve better reward performance?
This problem can be split into two sub-problems: One is how to reach the Pareto-optimal solution, defined in \Cref{definition:pareto-optimal}; and the other is how to search the Pareto front to reach desired constraint-satisfying location. In our algorithm, the first sub-problem is solved by running two adversarial stages alternately, the second sub-problem is solved by another projection stage. An overview of our algorithm is illustrated in \Cref{fig:overview}. 

\begin{definition}[Pareto-optimal solution] \label{definition:pareto-optimal}
    We call $\pi$ dominates $\pi'$, denoted as $\pi \succ \pi'$, if and only if $J_R(\pi) \geq J_R(\pi'), J_{C}(\pi) \leq J_{C}(\pi')$ and at least one inequality strictly holds. We call $\pi_\theta$ is a global Pareto-optimal solution if $\forall \pi_{\theta'} \in \Pi_{\theta}$, $\pi_{\theta'}$ does not dominate $\pi_\theta$.
\end{definition}

\begin{figure}[th!]
\begin{center}
\centerline{\includegraphics[width=0.7\columnwidth]{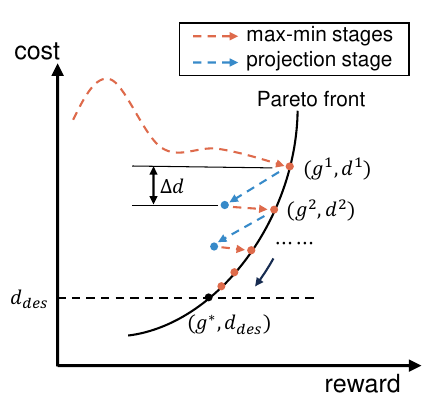}}
\caption{An illustration of our algorithm. The red dotted line represents the process of alternately iterations between max-reward and min-cost stages. The blue dotted line represents the projection stage. $\Delta d$ denotes the change of cost budget after projection stage explained in \Cref{section:projection}.}
\label{fig:overview}
\end{center}
\end{figure}

\subsection{Policy Update}
In this section, we explain in detail how to update the cost and reward alternately to reach the Pareto-optimal solution. Instead of using Pareto optimization or multi-objective optimization methods \cite{zhou2023gradient, huang2022constrained, xu2020prediction}, we propose a novel method to optimize reward and cost budget simultaneously in the perspective of constrained policy optimization, so that we can take advantage of theoretical and algorithmic foundations of Constrained RL. 

Specifically, the policy update process is divided to two alternating stages:
\newtheorem{stage}{ \bf Stage}
\begin{stage}[\textbf{max-reward}]
\label{stage:max-reward}
Given current cost budget $d^k$, maximize the reward return. $\pi^{d^k}_{max}$ denotes the solution of this problem, and $g^k=J_R(\pi^{d^k}_{max})$.
\begin{equation}
\begin{gathered}
    g^k = \max_\pi J_R(\pi) \\
    s.t. \;\; J_{C}(\pi) \leq d^k    
\end{gathered}
\end{equation}
\end{stage}
\begin{stage}[\textbf{min-cost}]
\label{stage:min-cost}
Given the reward budget $g^k$, minimize the cost return while keeping the reward return above the reward budget. $\pi^{g^k}_{min}$ is the solution of this problem, and $d^{k+1}=J_C(\pi^{g^k}_{min})$.
\begin{equation}
\begin{gathered}
    d^{k+1}=\min_\pi J_{C}(\pi) \\
    s.t. \;\; J_{R}(\pi) \geq g^k    
\end{gathered}
\end{equation}
\end{stage}

These two stages are solved alternately in the way of constrained policy update mentioned in Section \ref{section:constrained_policy_update}.
After solving the constrained problem in the current stage, the cost budget $d$ or reward budget $g$ is updated and will be used in the next stage. During this process, the generated trajectories of cost budgets $\{d^k\}_{k=0}^n$ and reward budgets $\{g^k\}_{k=0}^n$ can achieve monotonic updates according to \Cref{theorem:max-min}.
\begin{theorem} \label{theorem:max-min}
    By solving these two stages alternately, we get trajectories of policies $\{\pi^{d^0}_{max}, \pi^{g^0}_{min}, \pi^{d^1}_{max}, \pi^{g^1}_{min}, \cdots,\}$, cost budgets $\{d^k\}_{k=0}^n$, and reward budgets $\{g^k\}_{k=0}^n$. The worst case of update performance can be bounded by
    \begin{gather}
         g^k - g^{k-1} \geq  - \frac{\sqrt{2\delta}\gamma}{(1-\gamma)^2}(n_1 + n_2) \epsilon_R  \\
         d^{k+1} - d^k \leq  \frac{\sqrt{2\delta}\gamma}{(1-\gamma)^2}(n_2 + n_1) \epsilon_C
    \end{gather}
    where $n_1, n_2$ are the maximum iteration steps of max-reward stage and min-cost stage respectively. $\epsilon_R = \max_k \max_s |\mathbb{E}_{a\sim \pi_k}[A^{\pi_{k-1}}_R(s,a)]|,\; \epsilon_C = \max_k \max_s |\mathbb{E}_{a\sim \pi_k}[A^{\pi_{k-1}}_C(s,a)]|$. 
\end{theorem}
\begin{proof}
    The proof can be found in \Cref{section:adversarial_update_bound}.
\end{proof}

Different from most algorithms of Constrained RL, our algorithm additionally optimize the cost budget during the training process. Min-cost stage can be regarded as an adaptive technique to balance the effect of reward and cost signal during the training. When cost return is lower than cost budget, the budget will be adaptively decreased to a certain degree in which the policy can be guided mildly by moderate constraint violations.

However, these two stages alone can not guarantee that a solution satisfying desired constraint $J_C \leq d_{des}$ will eventually be obtained. The policy may converge to a local Pareto-optimal solution $(g^*, d^*)$ where $d^* > d_{des}$ or $d^* \ll d_{des}$. The former case can not satisfy the requirement of the task, and the latter will yield an over-conservative policy. To control the convergence direction of the cost budget, we develop an policy projection stage. 

\begin{algorithm}[t]
   \caption{Adversarial Constrained Policy Optimization}
\begin{algorithmic} \label{algo:acpo}
    \STATE Initialize $\pi^0 \in \Pi_\theta$, $d^0 > d_{des}$, $g^0 = 0$ and a trajectory buffer $\mathcal{B}$.
    \FOR{$k=0,1,2,\cdots$}
        \STATE Run $\pi^k$ and store trajectories in $\mathcal{B}$.
        \IF{$(d^k, g^k)$ converge}
            \IF{$d^k = d_{des}$}
                \STATE \textbf{Return} the policy $\pi^k$.
            \ENDIF
            \STATE Update policy $\pi^{k+1}$ and cost budget $d^{k+1}$ using Stage \ref{stage:projection}.
        \ELSE
            \STATE Update policy $\pi^{k+1}$ and budgets using Stage \ref{stage:max-reward} and \ref{stage:min-cost} alternately.
        \ENDIF

    \ENDFOR 

\end{algorithmic}
\end{algorithm}

\subsection{Policy Projection} \label{section:projection}

If the policy converge to a solution where the desired cost budget is not satisfied, then the following projection stage will be executed:
\begin{stage}[\textbf{projection}]
\label{stage:projection}
Let $D$ be a distance measure between two policies, $\pi_{old}$ is the converged policy after solving stage 1 and stage 2 alternately, $d_{old} = J_C(\pi_{old})$. We solve the following problem:
\begin{equation}
\begin{gathered}
    \pi_{new} = \arg\min_{\pi} D(\pi||\pi_{old})  \label{eq:distance_measurement} \\
    s.t. \;\; J_{C}(\pi) \leq d_{old} + \Delta d    
\end{gathered}
\end{equation}
Denote the solution of this problem as $\pi_{new}$, and $d_{new}=d_{old} + \Delta d$. $\Delta d$ can be controlled using a feedback controller, e.g. $\Delta d$ is calculated by the positional error between converged budget $d_{old}$ and desired value $d_{des}$:
\begin{equation}
    \Delta d = k_p (d_{des} - d_{old})  \label{eq:pd_control}
\end{equation}
where $k_p$ is a positive constant.
\end{stage}
If the policy converge to a solution where $d^* \ll d_{des}$, that means the policy is too conservative, we will only enlarge the cost budget by \Cref{eq:pd_control} without projection.

The aim of projection stage is to change the convergence direction of the policy along the Pareto front. When the policy converges while the desired budget does not satisfied, the policy will be projected to a tighter feasible region with shortest distance. After projection, the cost budget is updated and the policy is re-initialized, and the algorithm will continue to solve max-reward and min-cost stages alternately until the next convergence point appears.

\Cref{algo:acpo} shows the corresponding pseudocode of overall algorithm.

\section{Practical Implementation}
\label{section:practical_implementation}

\subsection{Solving the constrained problem}
Our proposed policy optimization strategy can be implemented by any any algorithm based on the framework of (\ref{eq:optimization_formulation}).
Instead of using second-order methods to solve this problem, we prefer to solve it by unconstrained gradient methods in order to decrease the computation complexity.
In this paper, we use interior-point policy optimization proposed by \cite{ipo} to solve (\ref{eq:optimization_formulation}). The interior point method transforms a constrained optimization problem into an unconstrained problem by constructing the penalty function. 
Using penalty function, (\ref{eq:optimization_formulation}) can be transformed into the following form:
\begin{equation}
\label{eq:ipo_optimization}
\begin{gathered}
    \pi_{k+1} = \arg\max_{\pi\in\Pi_{\theta}} \mathop{\mathbb{E}}\limits_{s\sim d_{\pi_k} \atop a\sim \pi}[A^{\pi_k}_R(s,a)]  +  \phi(\Tilde{J_C}(\pi))  \\
    \mathrm{s.t.}\;\; 
    \bar{D}_{KL}(\pi||\pi_k) \leq \delta
\end{gathered}
\end{equation}
where $\phi(x) = \frac{\log(-x)}{t}, t>0$ is a barrier function, and
\begin{equation}
    \Tilde{J_C}(\pi) =  J_C(\pi_k) + \frac{1}{1-\gamma}\mathop{\mathbb{E}}\limits_{s\sim d_{\pi_k} \atop a\sim \pi} [A^{\pi_k}_C(s,a)] - d
\end{equation}
The error between solutions of (\ref{eq:optimization_formulation}) and (\ref{eq:ipo_optimization}) can be bounded according to the following theorem:
\begin{theorem}
    \label{theorem:ipo_bound}
    Denote $\pi^*, \pi^*_{IPO}$ as the optimal solutions of (\ref{eq:optimization_formulation}) and (\ref{eq:ipo_optimization}) respectively. Assume $\pi^*_{IPO}$ is feasible for the constraint in (\ref{eq:optimization_formulation}), then the following bound holds
    \begin{equation}
     0 \leq \mathop{\mathbb{E}}\limits_{s\sim d_{\pi_k} \atop a\sim \pi^*}[A_R^{\pi_k}(s,a)] - \mathop{\mathbb{E}}\limits_{s\sim d_{\pi_k} \atop a\sim \pi^*_{IPO}}[A_R^{\pi_k}(s,a)] \leq \frac{1}{t}
    \end{equation}
\end{theorem}
\begin{proof}
    The proof can be found in \Cref{section:interior-optimal-bound}. 
\end{proof}

By using the form of (\ref{eq:ipo_optimization}) to solve Stage \ref{stage:max-reward} and Stage \ref{stage:min-cost} alternately, we can also get monotonic performance update bound like \Cref{theorem:max-min}:
\begin{proposition}
    By solving these two stages alternately using (\ref{eq:ipo_optimization}), we get trajectories of policies $\{\pi^{d^0}_{max}, \pi^{g^0}_{min}, \pi^{d^1}_{max}, \pi^{g^1}_{min}, \cdots,\}$, cost budgets $\{d^k\}_{k=0}^n$, and reward budgets $\{g^k\}_{k=0}^n$. The worst case of update performance can be bounded by:
    \begin{gather}
         g^k - g^{k-1} \geq  - \frac{n_1}{(1-\gamma)t} - \frac{\sqrt{2\delta}\gamma}{(1-\gamma)^2}(n_1 + n_2) \epsilon_R  \\
         d^{k+1} - d^k \leq  \frac{n_2}{(1-\gamma)t} + \frac{\sqrt{2\delta}\gamma}{(1-\gamma)^2}(n_2 + n_1) \epsilon_C
    \end{gather}
    where $n_1, n_2$ are the maximum iteration steps of max-reward stage and min-cost stage respectively. $\epsilon_R = \max_k \max_s |\mathbb{E}_{a\sim \pi_k}[A^{\pi_{k-1}}_R(s,a)]|,\; \epsilon_C = \max_k \max_s |\mathbb{E}_{a\sim \pi_k}[A^{\pi_{k-1}}_C(s,a)]|$.
\end{proposition}

In our experiments, the problem (\ref{eq:ipo_optimization}) is solved using PPO \cite{ppo}. Details about algorithm implementation are provided in \Cref{algo:details}. 

\subsection{Multi-constraint scenario}
For the multi-constraint scenario, the objective of min-cost stage can be expanded to the following form:
\begin{equation}
\begin{gathered}
    \min_\pi \sum_i J_{C_i}(\pi) \cdot \mathbb{I}(J_{C_i}(\pi) > d_{des}) \\
    s.t.\; J_R(\pi) \geq g^k
\end{gathered}
\end{equation}
where $\mathbb{I}(\cdot)$ denotes the indicator function. The objective function is the sum of cost returns which do not satisfy the desired budget. If all costs are below the desired budget, the min-cost stage will be skipped. 

\section{Experiments}
\label{section:experiments}

In this section, to evaluate the performance of our method, we conduct the following experiments:

\begin{enumerate}
    \item \textbf{Safety Gymnasium task:} Experiments on Safety Gymnasium environments are conducted to demonstrate that our method achieves better performance compared with other baselines while satisfying the constraints.

    \item \textbf{Quadruped locomotion task:} Experiments on legged robot with multi-constraints are conducted to demonstrate the effectiveness of our method on enhancing performances in high dimensional real-world scenario. 

    \item \textbf{Adapting budgets:} 
    We compare the performances of our algorithm with baseline method using curriculum budgets during the training process in the quadruped locomotion task, in order to showcase the benefits of alternating two stage iteration in our algorithm.

\end{enumerate}

\begin{table*}[t]
\centering
\caption{The performances of different algorithms in Car tasks. We conduct 5 experiments using random seeds on every environment.
Results are based on the episode returns of 10 evaluation rollouts. In these experiments, the desired cost budget is set to 25. 
The bold value of cost return is closest to the desired budget, and the bold value of the reward return represents the highest one among these solutions.}
\label{table:safetygym-car}

\begin{tabular}{c|cc|cc}
\hline
                      & \multicolumn{2}{c|}{\textbf{CarGoal1}}         & \multicolumn{2}{c}{\textbf{CarCircle1}}                  \\ \hline
\multicolumn{1}{l|}{} & $J_R$             & $J_C$             & $J_R$             & $J_C$                       \\ \hline
\textbf{Ours}         & $\mathbf{26.50} \pm 2.70$  & $\mathbf{26.90} \pm 0.14$  & $\mathbf{14.01} \pm 0.16$  & $\mathbf{24.30} \pm 2.37$              \\
\textbf{PPO-Lag}      & $20.68 \pm 5.08$  & $21.02 \pm 4.50$  & $13.11 \pm 2.65$  & $151.56 \pm 72.57$              \\
\textbf{IPO}          & $25.52 \pm 0.49$  & $27.51 \pm 2.10$  & $13.94 \pm 0.07$  & $26.75 \pm 2.18$              \\
\textbf{CPO}          & $19.79 \pm 1.13$  & $27.69 \pm 0.30$  & $13.81 \pm 0.65$  & $35.06 \pm 13.84$              \\
\textbf{CRPO}         & $22.61 \pm 0.75$  & $30.41 \pm 0.96$  & $13.38 \pm 0.37$  & $33.24 \pm 2.16$              \\
\hline
\end{tabular}
\end{table*}

\begin{table*}[h]
\centering
\caption{The performances of different algorithms in Point tasks. In these experiments, the desired cost budget is set to 25. 
The bold value of cost return is closest to the desired budget, and the bold value of the reward return represents the highest one among these solutions.}
\label{table:safetygym-point}

\begin{tabular}{c|cc|cc}
\hline
                      & \multicolumn{2}{c|}{\textbf{PointGoal1}}                  & \multicolumn{2}{c}{\textbf{PointCircle1}}                  \\ \hline
\multicolumn{1}{l|}{} & $J_R$                      & $J_C$                      & $J_R$                      & $J_C$                       \\ \hline
\textbf{Ours}         & $22.41 \pm 1.32$           & $26.72 \pm 0.90$           & $27.08 \pm 1.14$           & $1.12 \pm 0.77$              \\
\textbf{PPO-Lag}      & $\mathbf{23.26} \pm 0.48$  & $33.52 \pm 2.64$           & $24.78 \pm 10.66$          & $91.98 \pm 105.46$              \\
\textbf{IPO}          & $20.80 \pm 2.56$           & $27.88 \pm 3.44$           & $31.51 \pm 4.62$           & $6.24 \pm 7.29$              \\
\textbf{CPO}          & $17.21 \pm 3.11$           & $\mathbf{26.43} \pm 0.36$  & $32.69 \pm 0.46$           & $\mathbf{12.88} \pm 13.44$              \\
\textbf{CRPO}         & $17.78 \pm 0.22$           & $29.92 \pm 1.17$           & $\mathbf{32.75} \pm 15.80$ & $109.07 \pm 106.17$              \\
\hline
\end{tabular}
\end{table*}

\begin{figure*}[h!]
    \centering
    \begin{subfigure}
        \centering
        \includegraphics[width=\columnwidth]{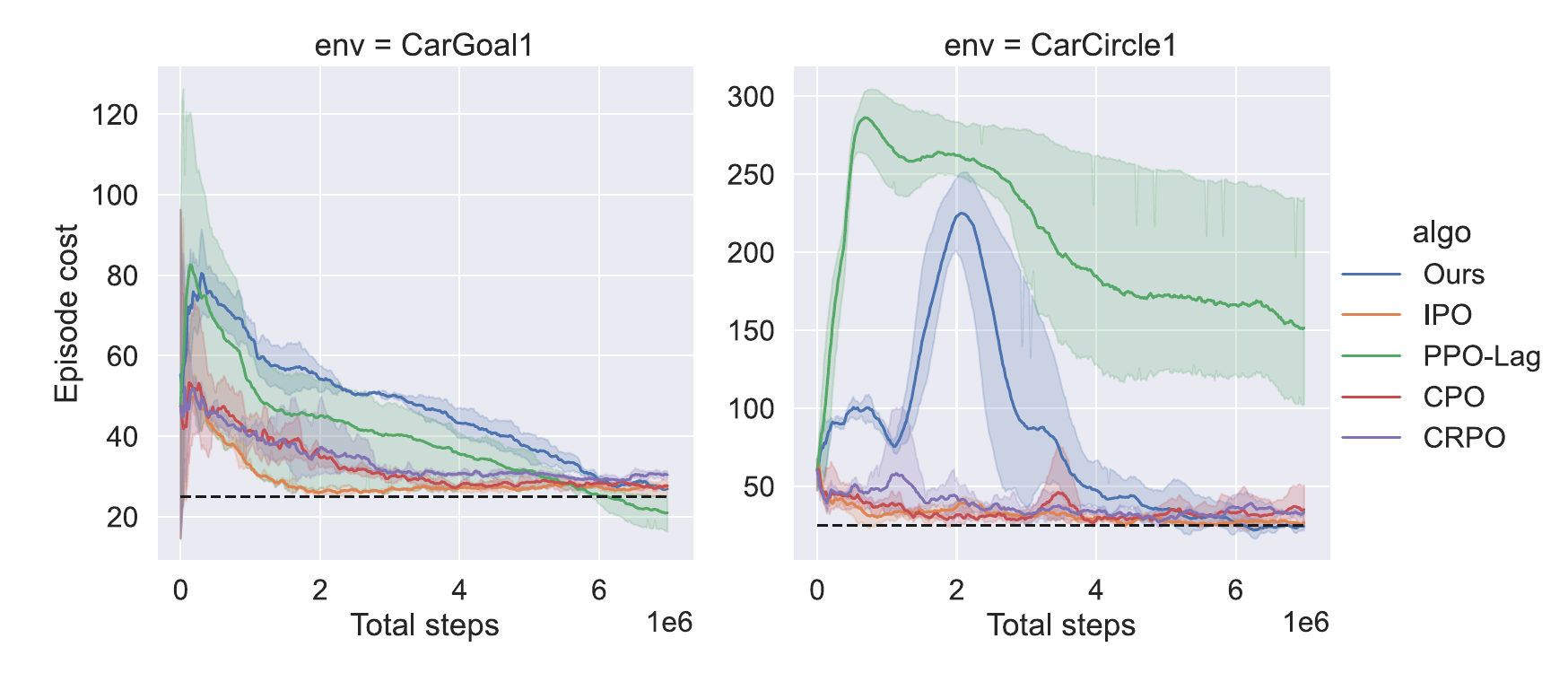}
    \end{subfigure} 
    \begin{subfigure}
        \centering
        \includegraphics[width=\columnwidth]{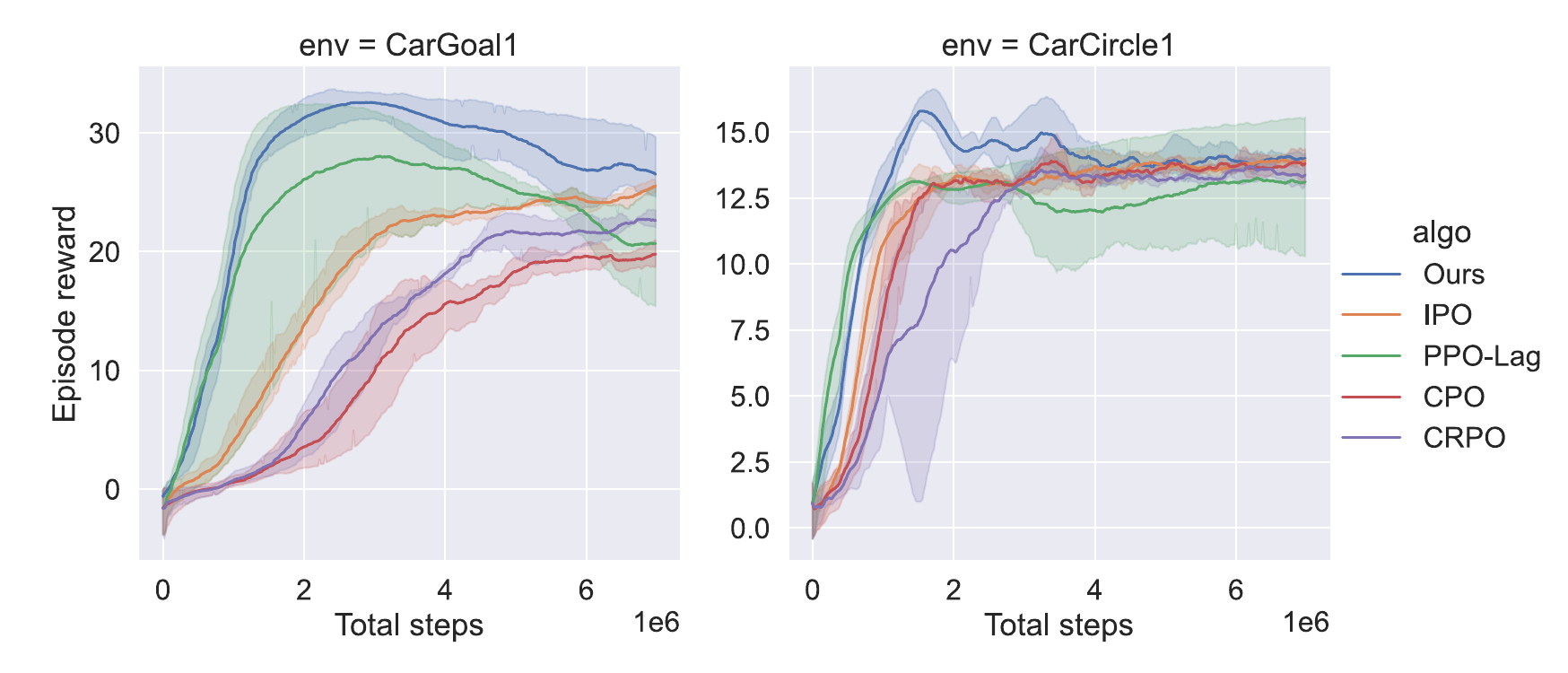}
    \end{subfigure}    
    
    \begin{subfigure}
        \centering
        \includegraphics[width=\columnwidth]{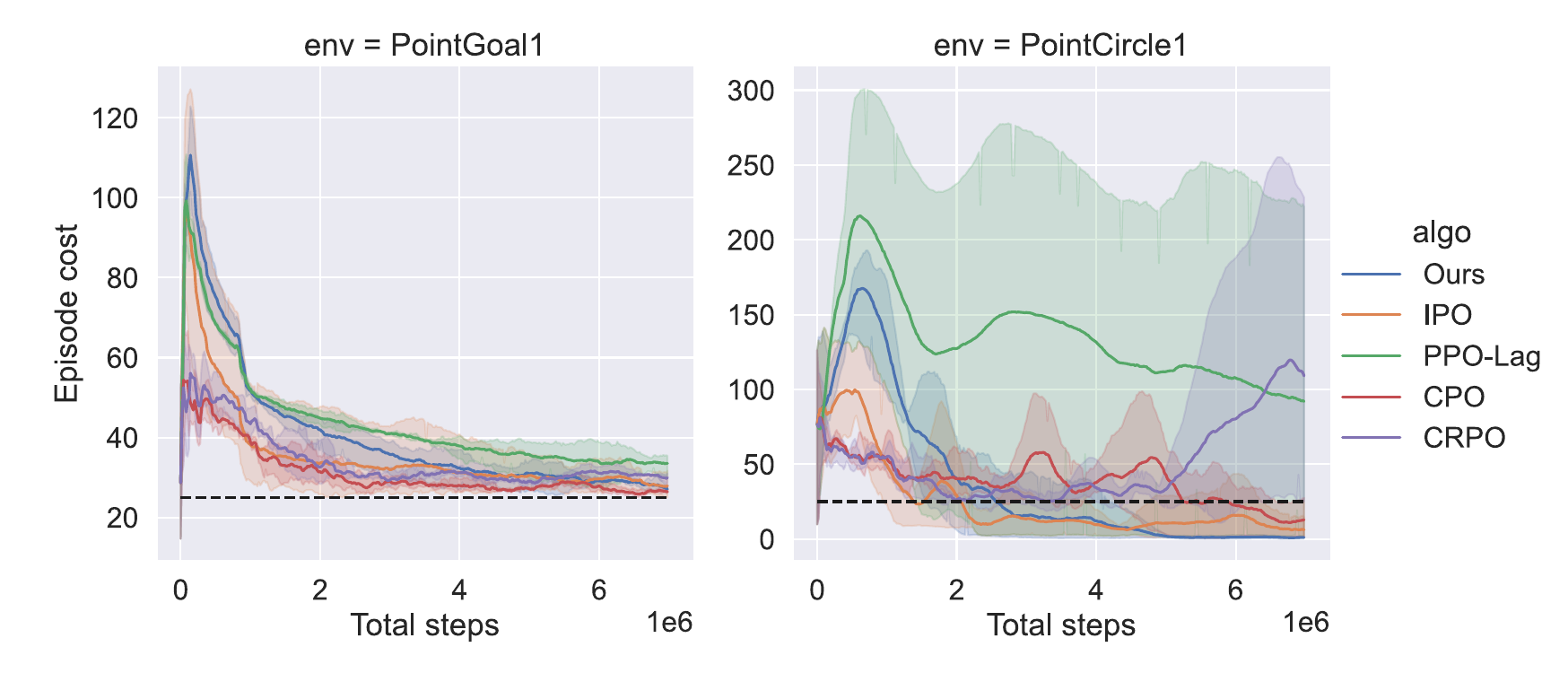}
    \end{subfigure} 
    \begin{subfigure}
        \centering
        \includegraphics[width=\columnwidth]{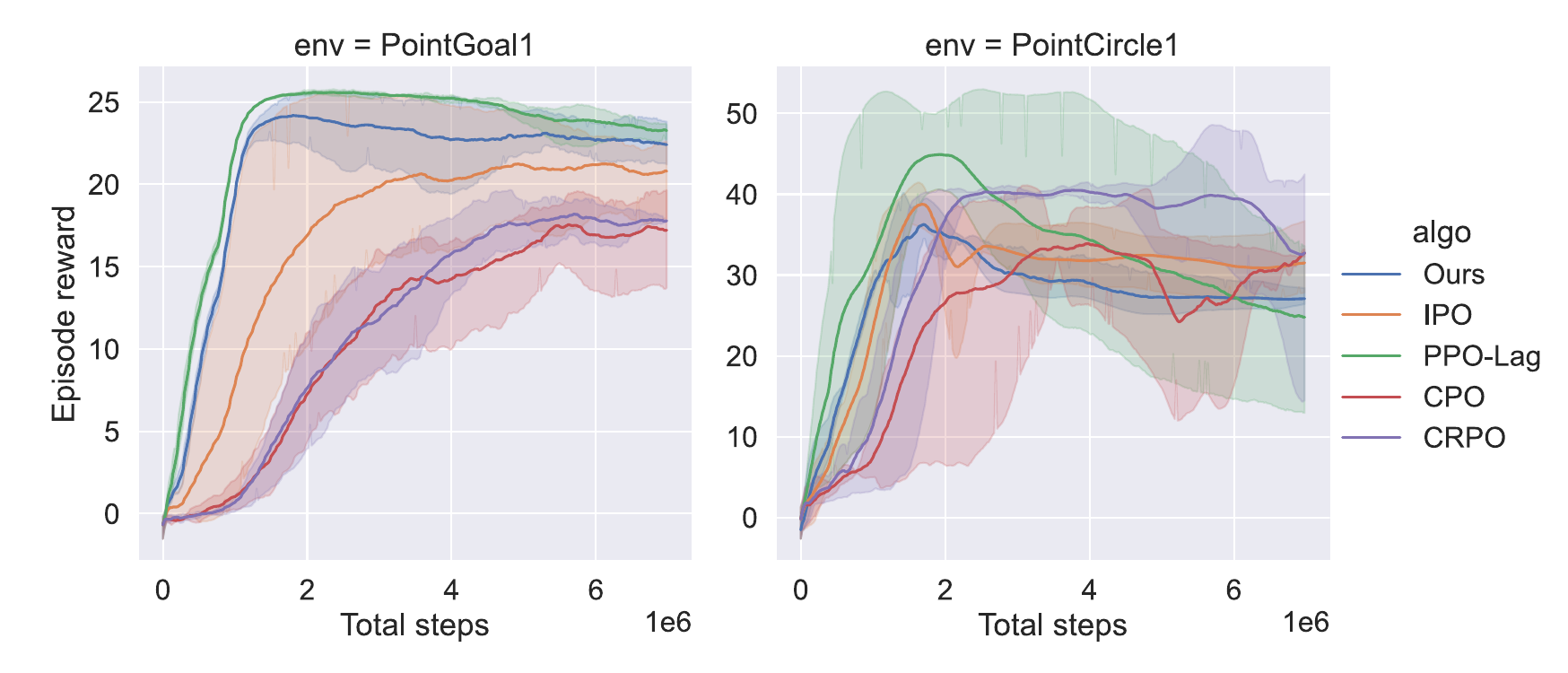}
    \end{subfigure}    
    \caption{The learning curve of different algorithms in Safety Gymnasium environments. The dashed line represents the desired cost budget. }
    \label{fig:safetygym}
\end{figure*}

\subsection{Environment Setup}

\subsubsection{Single constraint benchmark task}

We use Safety Gymnasium \cite{safety-gym} to compare the performance of our algorithm with other baselines. We conduct experiments on robotic navigation tasks (\texttt{Goal1}, \texttt{Circle1}) using two different agents (\texttt{Point}, \texttt{Car}). More details about these environments are provided in \Cref{appendix:safety-gym-env}.  

\subsubsection{Multi-constraint robotic task}

To demonstrate the effectiveness of our algorithm about performance improvement in high-dimensional complex tasks, we run our algorithm and other first-order methods in the locomotion task of Go1 quadruped robot, which is shown in \Cref{fig:leggedrobot-env}. 
In this environment, the goal of the robot is to track the desired velocity command while satisfying some task-specific constraints. At every control step, the agent outputs desired positions of 12 joint motors, and these desired positions are send to low-level PD controllers to generate desired torques of motors. The observation and action space are listed in \Cref{table:leggedrobot-observation_space}.

In this environment, the reward function is designed as the weighted sum of several parts: $R_t = \sum_{i=1}^{n_r} w_i R^i_t$, where $w_i, R^i_t$ denote the $i$th weight and reward respectively, which are listed in \Cref{table:leggedrobot-rewards}. The cost functions are designed as the following form: $C^i_t = B(c^i_t), \quad \forall i=1,\cdots,n_c $, where $c^i_t$ are one step cost functions, $B$ is a barrier function:
\begin{equation}
    B(c) = \frac{1+\mathrm{erf}(\frac{c-b}{\sigma})}{2}
\end{equation}
$b$ is the boundary of one step cost, $\sigma$ is the smoothing coefficient of the barrier function. Details about $c$ and its boundary $b$ are listed in \Cref{table:leggedrobot-costs}. 

\subsection{Baselines}

The baseline algorithms include: \textbf{IPO} \cite{ipo}: This algorithm uses interior point method to directly solve the constrained policy optimization problem without two stage iteration; \textbf{PPO-Lag} \cite{ray2019benchmarking}: This algorithm uses Lagrangian relaxation technique and optimizes the policy and the Lagrange multiplier simultaneously by gradient method; \textbf{CPO} \cite{cpo}: This algorithm approximates objective and constraint functions and solves a quadratic optimization problem like TRPO; \textbf{CRPO} \cite{crpo}:  This algorithm maximizes the reward and minimizes the cost violations alternately. It differs from our algorithm in that it uses unconstrained gradient method in each stage, whereas our algorithm models each stage as a constrained optimization sub-problem. 

We use implementations of these algorithms in Omnisafe \cite{omnisafe} project. Hyper-parameters of all algorithms are provided in \Cref{section:parameters_algorithms}. For a fair comparison, all algorithms use the same structure of networks, the same number of training samples, the same batch size.

\subsection{Results}

\begin{table*}[t]
\caption{The performances of different algorithms in the quadruped locomotion task. The desired cost budget $d$ of every cost type is set to 2. The episode reward is normalized relative to maximum episode length $1000$. The bold value represents the best solution which has the highest reward while satisfying desired constraints. IPO-C represents IPO using curriculum cost budgets.}
\label{table:leggedrobot-reward_and_cost}
\centering
\begin{tabular}{l|c|c|c|c}
\hline
                             & \textbf{Ours}             & \textbf{IPO-C}   & \textbf{IPO}     & \textbf{PPO-Lag}           \\ \hline
Episode reward               & $\mathbf{0.93} \pm 0.02$  & $0.86 \pm 0.03$  & $0.85 \pm 0.03$  & $0.81 \pm 0.08$            \\
Episode cost of action\_rate & $\mathbf{1.44} \pm 0.12$  & $1.52 \pm 0.13$  & $0.94 \pm 0.22$  & $0.29 \pm 0.09$            \\
Episode cost of ang\_vel\_xy & $\mathbf{1.69} \pm 0.13$  & $1.63 \pm 0.12$  & $1.47 \pm 0.10$  & $1.54 \pm 0.52$            \\
Episode cost of lin\_vel\_z  & $\mathbf{0.86} \pm 0.11$  & $0.60 \pm 0.20$  & $0.64 \pm 0.33$  & $1.04 \pm 0.27$            \\ \hline
\end{tabular}
\end{table*}

\begin{figure*}[t]
    \centering
    \includegraphics[width=\textwidth]{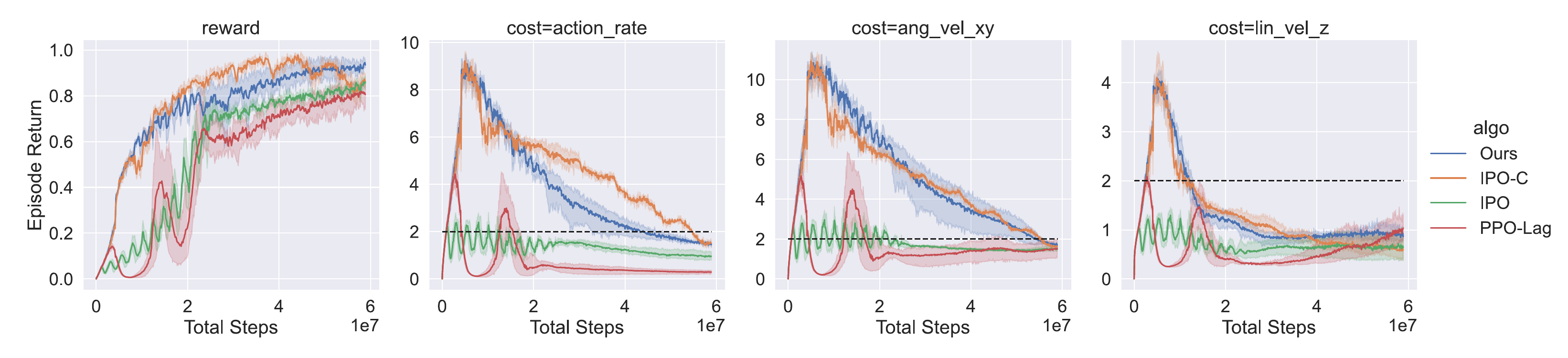}
    \caption{Episode reward and cost return of different algorithms in the quadruped locomotion task. Results are the mean value of 5 training experiments using random seeds. The dashed line represents the desired cost budget = 2. The episode reward in the figure are normalized relative to maximum episode length $1000$. }
    \label{fig:leggedrobot-reward_and_cost}
\end{figure*}

\subsubsection{Safety Gymnasium Results}
We compare the performances of our algorithm with other baselines in terms of constraint satisfaction and final reward returns. \Cref{table:safetygym-car}, \Cref{table:safetygym-point} and \Cref{fig:safetygym} demonstrate results of different algorithms. The higher the reward and the closer the cost is to the desired budget, the better the solution. 

In the CarGoal1 and CarCircle1 tasks, our algorithm outperforms other baselines in both tasks. With the same number of training samples, our algorithm has the highest reward and the closest cost value to the desired budget. In the PointGoal1 task, our algorithm outperforms IPO, CRPO on both reward and cost metrics. The cost error between ours and CPO is negligible, but the reward of ours is significantly higher than CPO's. The reward of ours is 30\% higher than CPO's. 
In the PointCircle1 task, our algorithm only outperforms PPO-Lag on both reward and cost metrics. Compared to the others, our algorithm gets a lower cost value, which results in a lower reward.  
We believe the reason is that this task is relatively simple and encouraging exploration by adapting budgets is not a significant impact on performance improvement. 

Thanks to large cost budget at the early phase of training, the agent is able to fully explore and boost the reward value, so the reward curve of our algorithm can increase faster than others at the early training stage. This phenomenon is similar with the result of PPO-Lag method. Moreover, our algorithm can achieve better cost results than PPO-Lag. The reason is that: When the policy get stuck in a local minima where the cost is above the desired budget, the projection stage takes over policy optimization process and project the policy towards the direction of cost decline. It helps to re-initialize the policy to the next iteration of min-cost and max-reward stages.

\subsubsection{Quadruped Locomotion Task Results}
In this task, we choose IPO, PPO-Lag as baseline algorithms because of their low computational complexity and ability to tackle multi constraints. This task involves totally 8 costs, but only 3 costs are strongly relative to the reward and others always have near zero values during the training. Therefore, Table \ref{table:leggedrobot-reward_and_cost} and Figure \ref{fig:leggedrobot-reward_and_cost} only demonstrate these 3 costs and the reward. Detailed results are provided in \Cref{section:leggedrobot-detailed-results}.

In this experiment, results show that all algorithms satisfy these constraints, so we only compare the reward returns. The higher the reward, the better the solution is. Results demonstrate that our algorithm achieves the highest reward among baselines. Our algorithm exhibits a 9\% increase in episode reward compared to IPO, and a 15\% increase compared to PPO-Lag, which demonstrates the efficiency of our algorithm in complex tasks.

\subsubsection{Comparison with curriculum budgets}
In this section, we aim to investigate the impact of alternating two-stage iteration on the performance enhancement. We compare the performance of our algorithm with the IPO using curriculum cost budgets, denoted as IPO-C. IPO-C uses IPO to solve the constrained optimization problem, and changes the cost budget gradually from an initial large value to the desired value during the training process. \Cref{fig:leggedrobot-budgetcost} illustrates the trajectories of cost budgets of IPO-C and our algorithm during training.

\Cref{table:leggedrobot-reward_and_cost} and \Cref{fig:leggedrobot-reward_and_cost} shows that our algorithm outperforms IPO-C with 8\% higher reward. Due to the loose constraints, IPO-C has a high reward during the early and middle phase of the training. As constraints are gradually tightened, the reward curve of IPO-C undergoes a drop, leading to a lower final results. While our algorithm achieves a near-monotonic increase of the reward. This can be attributed to the strategy of alternating two adversarial stages. In the min-cost stage, we not only minimize the cost return, but also maintain the current reward calculated by the max-reward stage, which can effectively mitigate the decrease of reward caused by cost minimization. Moreover, cost budgets can be adjusted adaptively by alternating two stages, which is preferable to the predefined budget trajectory because we usually do not know the priori information of the environment.

\section{Conclusion}

In this paper, we propose a new constrained policy optimization method which can optimize the reward and adapt cost budgets simultaneously. Experiments on toy environments and complex robotic tasks are conducted to
demonstrate the effectiveness of our method. Compared to the state-of-the-art methods, our method can achieve better reward return while satisfying constraints. In this paper, we use interior-point method to implement our algorithm, in the future, we aim to apply our method to other constrained policy optimization framework, e.g. Lagrangian relaxation methods, trust region methods. Moreover, we aim to provide theoretical guarantees of the projection stage and give a tighter bound of the update performance when alternating two stages.

\section{Impact Statements}
This paper presents work whose goal is to advance the field of Machine Learning. There are many potential societal consequences of our work, none which we feel must be specifically highlighted here.



\nocite{yu2019convergent}

\bibliography{example_paper}
\bibliographystyle{icml2024}

\newpage
\appendix
\onecolumn

\section{Adversarial Update Bound} \label{section:adversarial_update_bound}
In this section, we will prove \Cref{theorem:max-min} about the worst case of performance update through alternate iteration between Stage \ref{stage:max-reward} and Stage \ref{stage:min-cost}.

\begin{lemma}[Trust region update performance bound \cite{cpo}] \label{lemma:cpo-worst-case}
    Suppose $\pi_k$ and $\pi_{k+1}$ are related by (\ref{eq:optimization_formulation}), and $\pi_k\in\Pi_\theta$. A lower bound on the policy performance difference between $\pi_k$ and $\pi_{k+1}$ is
    \begin{equation}
        J_R(\pi_{k+1}) - J_R(\pi_k) \geq 
        - \frac{\sqrt{2\sigma}\gamma \epsilon^{\pi_{k+1}}_R}{(1-\gamma)^2}
    \end{equation}
    where $\epsilon_R^{\pi_{k+1}}=\max_s |\mathbb{E}_{a\sim \pi_{k+1}}[A^{\pi_k}_R(s,a)]|$. 

\end{lemma}

\begin{lemma}
\label{lemma:max-reward-bound}
Assume max-reward stage has totally $n_1$ updates, $\pi_i, \forall i=1,\cdots,n_1$ denote the solutions of these updates, and the initial policy is $\pi_0 = \pi^{g^{k-1}}_{min}$, which is the converged policy of the last min-cost stage. The $i$-th update can be formulated as the follow:
\begin{equation}
\label{eq:max-reward_optimization}
\begin{gathered}
    \pi_{i+1} = \arg\max_{\pi \in \Pi_\theta} \mathop{\mathbb{E}}\limits_{s\sim d_{\pi_i} \atop a \sim \pi} [ A_R^{\pi_i}(s,a) ]  \\
    \mathrm{s.t.} \;\; J_C(\pi_i) + \frac{1}{1-\gamma}\mathop{\mathbb{E}}\limits_{s\sim d_{\pi_i} \atop a \sim \pi} [ A_C^{\pi_i}(s,a) ] \leq d^k  \\
    \bar{D}_{KL}(\pi || \pi_i) \leq \sigma
\end{gathered}
\end{equation}
Then the following bound holds:
\begin{equation}
    J_R(\pi_{n_1}) - J_R(\pi_0) \geq  - \frac{\sqrt{2\sigma}\gamma }{(1-\gamma)^2} \sum_{j=1}^{n_1} \epsilon_R^{\pi_j} 
\end{equation}
where $\epsilon_R^{\pi_j}=\max_s |\mathbb{E}_{a\sim \pi_j}[A_R^{\pi_{j-1}}(s,a)]|$.    
\end{lemma}

\begin{proof}
    This bound directly follows from Lemma \ref{lemma:cpo-worst-case}.
\end{proof}

\begin{lemma}
\label{lemma:min-cost-bound}
Assume min-cost stage has totally $n_2$ updates, $\pi_i, \forall i=1,\cdots,n_2$ denote the solutions of these updates, and the initial policy is $\pi_0 = \pi^{d^k}_{max}$, which is the converged policy of the last max-reward stage. The $i$-th update can be formulated as the follow:
\begin{equation}
\label{eq:min-cost_optimization}
\begin{gathered}
    \pi_{i+1} = \arg\min_{\pi \in \Pi_\theta} \mathop{\mathbb{E}}\limits_{s\sim d_{\pi_i} \atop a \sim \pi} [ A_C^{\pi_i}(s,a) ]  \\
    \mathrm{s.t.} \;\; J_R(\pi_i) + \frac{1}{1-\gamma}\mathop{\mathbb{E}}\limits_{s\sim d_{\pi_i} \atop a \sim \pi} [ A_R^{\pi_i}(s,a) ] \geq g^k  \\
    \bar{D}_{KL}(\pi || \pi_i) \leq \sigma
\end{gathered}
\end{equation}
Then the following bound holds:
\begin{equation}
    J_R(\pi_{n_2}) \geq g^k - \frac{\sqrt{2\sigma}\gamma }{(1-\gamma)^2} \sum_{j=1}^{n_2} \epsilon_R^{\pi_j} 
\end{equation}
where $\epsilon_R^{\pi_j}=\max_s |\mathbb{E}_{a\sim \pi_j}[A_R^{\pi_{j-1}}(s,a)]|$.    
\end{lemma}

\begin{proof}
According to Lemma \ref{lemma:cpo-worst-case}, we have
\begin{align}
    &
    J_R(\pi_1) - J_R(\pi_0) \geq \frac{1}{1-\gamma} \mathop{\mathbb{E}}\limits_{s\sim d_{\pi_{0}} \atop a \sim \pi_1}[A_R^{\pi_0}(s,a)] - \frac{\sqrt{2\sigma}\gamma }{(1-\gamma)^2} \epsilon_R^{\pi_1} \\
    \Rightarrow \quad&
    J_R(\pi_1) \geq g^k + \frac{1}{1-\gamma} \mathop{\mathbb{E}}\limits_{s\sim d_{\pi_{0}} \atop a \sim \pi_1}[A_R^{\pi_0}(s,a)] - \frac{\sqrt{2\sigma}\gamma }{(1-\gamma)^2} \epsilon_R^{\pi_1} \\
    \Rightarrow \quad&
    J_R(\pi_1) \geq g^k - \frac{\sqrt{2\sigma}\gamma }{(1-\gamma)^2} \epsilon_R^{\pi_1}
\end{align}

We focus on the worst case, i.e., we assume $J_R(\pi_1)$ is close to its lower bound. Under this circumstance, we have
\begin{align}
    &
    J_R(\pi_2) - J_R(\pi_1) \geq \frac{1}{1-\gamma} \mathop{\mathbb{E}}\limits_{s\sim d_{\pi_{1}} \atop a \sim \pi_2}[A_R^{\pi_1}(s,a)] - \frac{\sqrt{2\sigma}\gamma }{(1-\gamma)^2} \epsilon_R^{\pi_2} \\    
    \Rightarrow \quad &
    J_R(\pi_2) \geq g^k + \frac{1}{1-\gamma} \mathop{\mathbb{E}}\limits_{s\sim d_{\pi_{1}} \atop a \sim \pi_2}[A_R^{\pi_1}(s,a)] - \frac{\sqrt{2\sigma}\gamma }{(1-\gamma)^2} (\epsilon_R^{\pi_1} + \epsilon_R^{\pi_2})   
\end{align}
According to (\ref{eq:min-cost_optimization}), Because $J_R(\pi_1)$ is at the lower bound and the algorithm aims to satisfy the constraint inequality, it is necessary that 
\begin{equation}
    \frac{1}{1-\gamma} \mathop{\mathbb{E}}\limits_{s\sim d_{\pi_{1}} \atop a \sim \pi_2}[A_R^{\pi_1}(s,a)] \geq 0
\end{equation}
then we get
\begin{equation}
    J_R(\pi_2) \geq g^k - \frac{\sqrt{2\sigma}\gamma }{(1-\gamma)^2} (\epsilon_R^{\pi_1} + \epsilon_R^{\pi_2})  
\end{equation}

Repeating the above derivation in $\pi_i,i=3,\cdots,n_2$, we can finally get
\begin{equation}
    J_R(\pi_{n_2}) \geq g^k - \frac{\sqrt{2\sigma}\gamma }{(1-\gamma)^2} \sum_{j=1}^{n_2} \epsilon_R^{\pi_j} 
\end{equation}

\end{proof}

Assume max-reward stage and min-cost stage have $n_1,n_2$ maximum number of iterations respectively. According to Lemma \ref{lemma:max-reward-bound} and Lemma \ref{lemma:min-cost-bound}, we have
\begin{align}
    & g^k - g^{k-1}   \\ 
    = \quad& J_R(\pi_{max}^{d^k}) - g^{k-1}  \\  
    \geq \quad& 
    J_R(\pi_{max}^{d^k}) - J_R(\pi_{min}^{g^{k-1}}) - \frac{\sqrt{2\sigma}\gamma }{(1-\gamma)^2} \sum_{j=1}^{n_2} \epsilon_R^{\pi_j^g}  \\
    \geq \quad&
     - \frac{\sqrt{2\sigma}\gamma }{(1-\gamma)^2} \sum_{j=1}^{n_1} \epsilon_R^{\pi_j^d} - \frac{\sqrt{2\sigma}\gamma }{(1-\gamma)^2} \sum_{j=1}^{n_2} \epsilon_R^{\pi_j^g} \\
     \geq \quad&
     - \frac{\sqrt{2\sigma}\gamma }{(1-\gamma)^2} (n_1 + n_2)\epsilon_R 
\end{align}
We use superscripts $d, g$ to distinguish different stages: $\pi_j^g, j=1,\cdots, n_2$ denote the min-cost stage; $\pi_j^d, j=1,\cdots,n_1$ denote the max-reward stage. And $\epsilon_R$ is the maximum value among $\epsilon_R^{\pi_j^d}, \epsilon_R^{\pi_j^g}, \forall j$.

Similarly, we have
\begin{align}
    & d^{k+1} - d^{k}  \\
    = \quad& J_C(\pi_{min}^{g^k}) - d^k  \\
    \leq \quad& J_C(\pi_{min}^{g^k}) - J_C(\pi_{max}^{d^k}) + \frac{\sqrt{2\sigma}\gamma }{(1-\gamma)^2} \sum_{j=1}^{n_1} \epsilon_C^{\pi_j^d} \\
    \leq \quad& \frac{\sqrt{2\sigma}\gamma }{(1-\gamma)^2} \sum_{j=1}^{n_2} \epsilon_C^{\pi_j^g} + \frac{\sqrt{2\sigma}\gamma }{(1-\gamma)^2} \sum_{j=1}^{n_1} \epsilon_C^{\pi_j^d} \\
    \leq \quad& \frac{\sqrt{2\sigma}\gamma }{(1-\gamma)^2}(n_1 + n_2) \epsilon_C
\end{align}
where $\epsilon_C$ is the maximum value among $\epsilon_C^{\pi_j^d}, \epsilon_C^{\pi_j^g}, \forall j$.

\section{Interior Point Method Optimal Bound} \label{section:interior-optimal-bound}
In this section, we will prove Theorem \ref{theorem:ipo_bound}.

We have the following problem:
\begin{equation}
\begin{gathered}
    \max_{\pi\in\Pi_{\theta}} \mathop{\mathbb{E}}\limits_{s\sim d_{\pi_k} \atop a\sim \pi}[A^{\pi_k}(s,a)] \\
    s.t.\;\; J_{C}(\pi_k) + \frac{1}{1-\gamma}\mathop{\mathbb{E}}\limits_{s\sim d_{\pi_k} \atop a\sim \pi} [A^{\pi_k}_{C}(s,a)] \leq d \\
    \bar{D}_{KL}(\pi|\pi_k) \leq \delta    
\end{gathered}
\end{equation}
Its Lagrange function is:
\begin{gather}
    \mathcal{L}= \mathop{\mathbb{E}}\limits_{s\sim d_{\pi_k} \atop a\sim \pi}[A^{\pi_k}(s,a)] -\beta(\bar{D}_{KL}(\pi|\pi_k) - \delta) -\lambda \widetilde{J}_C  \\
    \mathrm{where}\;\; \widetilde{J}_C = J_{C}(\pi_k) + \frac{1}{1-\gamma}\mathop{\mathbb{E}}\limits_{s\sim d_{\pi_k} \atop a\sim \pi} [A^{\pi_k}_{C}(s,a)] - d
\end{gather}
$\beta \geq 0, \lambda \geq 0$ are Lagrange multipliers.
We can reformulate the problem into the following primal form
\begin{equation}
    x^* = \max_{\pi} \min\limits_{\beta\geq 0 \atop \lambda \geq 0} \mathcal{L}
\end{equation}
and assume the optimal solution is: $(\pi^*, \beta^*, \lambda^*)$.

Define the dual function as
\begin{equation}
    f(\beta,\lambda) = \max_\pi \mathcal{L}
\end{equation}
and the dual problem as
\begin{equation}
\begin{gathered}
    p^* = \min_{\beta,\lambda} f(\beta,\lambda) \\
    s.t.\;\; \beta \geq 0 \\
    \lambda \geq 0    
\end{gathered}
\end{equation}
Assume the optimal solution of the dual problem is $(\pi^*_{d}, \beta^*_d, \lambda^*_d)$. 

According to the duality gap property, we have
\begin{equation}
    x^* \leq p^*
\end{equation}
In our implementation, we use the log barrier function to model the $\widetilde{J}_C \leq 0$ constraint and transform it into the objective. The corresponding optimization problem is
\begin{equation}
    \max_\pi \min\limits_{\beta\geq 0} \mathop{\mathbb{E}}\limits_{s\sim d_{\pi_k} \atop a\sim \pi}[A^{\pi_k}(s,a)] -\beta(\bar{D}_{KL}(\pi|\pi_k)-\delta) + \frac{\log(-\widetilde{J}_C)}{t}
    \label{eq:ipo}
\end{equation}
Assume (\ref{eq:ipo}) has an optimal solution $(\pi^*_{IPO}, \beta^*_{IPO})$ which is feasible. And $\pi^*_{IPO}$ satisfies
\begin{gather}
    \nabla L(\pi^*_{IPO}, \beta^*_{IPO}) + \frac{1}{\widetilde{J}_C(\pi^*_{IPO})t} \nabla \widetilde{J}_C(\pi^*_{IPO}) = 0 \\
    \mathrm{where}\; L(\pi,\beta) = \mathop{\mathbb{E}}\limits_{s\sim d_{\pi_k} \atop a\sim \pi}[A^{\pi_k}(s,a)] -\beta(\bar{D}_{KL}(\pi|\pi_k)-\delta)
\end{gather}
Make
\begin{equation}
    \lambda' = -\frac{1}{\widetilde{J}_C(\pi^*_{IPO})t}
\end{equation}
then we have
\begin{gather}
    f(\beta^*_{IPO}, \lambda') = \max_\pi L(\pi,\beta^*_{IPO}) + \frac{1}{\widetilde{J}_C(\pi^*_{IPO})t} \widetilde{J}_C(\pi) \\
    = L(\pi^*_{IPO}, \beta^*_{IPO}) + \frac{1}{t} 
\end{gather}
Because $(\pi^*_d, \beta^*_d, \lambda^*_d)$ are the optimal solution of the dual problem, then
\begin{gather}
    f(\beta^*_d, \lambda^*_d) \leq f(\beta^*_{IPO}, \lambda') \\
    L(\pi^*_d, \beta^*_d) + \lambda^*_d \widetilde{J}_C(\pi^*_d) \leq L(\pi^*_{IPO}, \beta^*_{IPO}) + \frac{1}{t}
\end{gather}
According to the duality gap, we have
\begin{gather}
    L(\pi^*, \beta^*) + \lambda^* \widetilde{J}_C(\pi^*) \leq L(\pi^*_{IPO}, \beta^*_{IPO}) + \frac{1}{t} 
\end{gather}
Because of the property of complementary slackness, we can further get
\begin{gather}
    \mathop{\mathbb{E}}\limits_{s\sim d_{\pi_k} \atop a\sim \pi^*}[A^{\pi_k}(s,a)]  \leq \mathop{\mathbb{E}}\limits_{s\sim d_{\pi_k} \atop a\sim \pi^*_{IPO}}[A^{\pi_k}(s,a)] + \frac{1}{t}
\end{gather}
Because $\pi^*_{IPO}$ is feasible, therefore
\begin{gather}
    \mathop{\mathbb{E}}\limits_{s\sim d_{\pi_k} \atop a\sim \pi^*}[A^{\pi_k}(s,a)] - \mathop{\mathbb{E}}\limits_{s\sim d_{\pi_k} \atop a\sim \pi^*_{IPO}}[A^{\pi_k}(s,a)] \geq 0
\end{gather}
Finally, we have
\begin{gather}
     0 \leq \mathop{\mathbb{E}}\limits_{s\sim d_{\pi_k} \atop a\sim \pi^*}[A^{\pi_k}(s,a)] - \mathop{\mathbb{E}}\limits_{s\sim d_{\pi_k} \atop a\sim \pi^*_{IPO}}[A^{\pi_k}(s,a)] \leq \frac{1}{t}
\end{gather}

\section{Adversarial Update Bound using Interior Point Method} \label{section:update-bound-ipo}
Denote $\pi^*_{IPO}$ as the optimal solution of the problem \ref{eq:ipo_optimization}. According to Lemma \ref{lemma:CPO_bound} and Theorem \ref{theorem:ipo_bound}, we have
\begin{align}
    &
    J(\pi^*_{IPO}) - J(\pi_k) \geq \frac{1}{1-\gamma} \mathop{\mathbb{E}}\limits_{s\sim d_{\pi_k} \atop a\sim \pi^*_{IPO}}[A^{\pi_k}(s,a)] - \frac{\sqrt{2}\gamma \epsilon^{\pi^*_{IPO}}}{(1-\gamma)^2} \sqrt{\mathbb{E}_{s\sim d_{\pi_k}}[D_{KL}(\pi^*_{IPO}|\pi_k)[s]]}  \\
    \Rightarrow \quad&
    J(\pi^*_{IPO}) - J(\pi_k) \geq \frac{1}{1-\gamma} \mathop{\mathbb{E}}\limits_{s\sim d_{\pi_k} \atop a\sim \pi^*_{IPO}}[A^{\pi_k}(s,a)] - \frac{\sqrt{2\delta}\gamma \epsilon^{\pi^*_{IPO}}}{(1-\gamma)^2} \\
    \Rightarrow \quad&
    J(\pi^*_{IPO}) - J(\pi_k) \geq \frac{1}{1-\gamma} \mathop{\mathbb{E}}\limits_{s\sim d_{\pi_k} \atop a\sim \pi^*}[A^{\pi_k}(s,a)] - \frac{1}{(1-\gamma)t} - \frac{\sqrt{2\delta}\gamma \epsilon^{\pi^*_{IPO}}}{(1-\gamma)^2} \\
    \Rightarrow \quad&
    J(\pi^*_{IPO}) - J(\pi_k) \geq - \frac{1}{(1-\gamma)t} - \frac{\sqrt{2\delta}\gamma \epsilon^{\pi^*_{IPO}}}{(1-\gamma)^2}
\end{align}
Then we can replace the bound in Lemma \ref{lemma:max-reward-bound} as
\begin{equation}
    J_R(\pi_{n_1}) - J_R(\pi_0) \geq -\frac{n_1}{(1-\gamma)t} - \frac{\sqrt{2\sigma}\gamma }{(1-\gamma)^2} \sum_{j=1}^{n_1} \epsilon_R^{\pi_j} 
\end{equation}
Similarly to the derivation of \Cref{theorem:max-min}, we can get
\begin{align}
    &
    g^k - g^{k-1}   \\ 
    \geq \quad& J(\pi^{d^k}_{max}) - J(\pi_{min}^{g^{k-1}}) - \frac{\sqrt{2\delta}\gamma}{(1-\gamma)^2}n_2 \epsilon_R \\
    \geq \quad& - \frac{n_1}{(1-\gamma)t} - \frac{\sqrt{2\delta}\gamma}{(1-\gamma)^2}n_1\epsilon_R - \frac{\sqrt{2\delta}\gamma}{(1-\gamma)^2} n_2 \epsilon_R \\
    = \quad& - \frac{n_1}{(1-\gamma)t} - \frac{\sqrt{2\delta}\gamma}{(1-\gamma)^2}(n_1+ n_2) \epsilon_R
\end{align}
And
\begin{align}
    &
    d^{k+1} - d^k \\
    \leq \quad& 
    J_C(\pi^{g^k}_{min}) - J_C(\pi^{d^k}_{max}) + \frac{\sqrt{2\delta}\gamma}{(1-\gamma)^2}n_1 \epsilon_C \\
    \leq \quad&
    \frac{n_2}{(1-\gamma)t} +
    \frac{\sqrt{2\delta}\gamma}{(1-\gamma)^2}n_2 \epsilon_C + \frac{\sqrt{2\delta}\gamma}{(1-\gamma)^2}n_1 \epsilon_C \\
    = \quad& \frac{n_2}{(1-\gamma)t} + \frac{\sqrt{2\delta}\gamma}{(1-\gamma)^2}(n_2 + n_1) \epsilon_C    
\end{align}

\section{Algorithm Details}
\Cref{algo:details} demonstrates the details about the implementation using interior-point method. The function \texttt{UpdateBudgets} in \Cref{algo:details} are provided in \Cref{algo:update_budgets}.
\begin{algorithm}[h]
   \caption{Algorithm implementation using interior-point method}
   \label{algo:details}
\begin{algorithmic} 
    \STATE \textbf{Initialize}: policy network parameter $\theta_0$, reward value network parameter $w_0^r$, cost value network parameter $w^{c}_0$, the data buffer $\mathcal{B}$, reward return queue $\mathcal{D}_R$, cost return queue $\mathcal{D}_C$, stage flag $f_0=0$, reward budget $g^0$, cost budget $d^0$.
    \FOR{$k=1,2,\cdots$}
        \STATE Collect batch data of $\{(s_i, a_i, s_{i+1}, r_i, c_i)\}_{i=0}^N$ using current policy $\pi_{\theta_k}$.
        \STATE Estimate reward and cost return: $\hat{J}_R, \hat{J}_C$ and append to queues $\mathcal{D}_R, \mathcal{D}_C$.
        \STATE Calculate GAE: $\{A_R(s_i,a_i), A_C(s_i,a_i)\}_{i=0}^N$ and target value: $\{V^{tar}_R(s_i), V^{tar}_C(s_i) \}_{i=0}^N$.
        \STATE Store data: $\mathcal{B} \leftarrow \{s_i, a_i, A_R(s_i,a_i), A_C(s_i,a_i), V_R^{tar}(s_i), V_C^{tar}(s_i)\}_{i=0}^N$.
        \STATE Update budgets and stage flag: $f_{k}, d^k, g^k \leftarrow \mathrm{UpdateBudgets}(\mathcal{D}_R, \mathcal{D}_C ) $
        \FOR{each mini-batch of size $B$ from the buffer}
            \STATE 
            \begin{equation*}
                w^r_{k} \leftarrow \arg\min_w \sum_{i=1}^B (V_R(s_i;w) - V_R^{tar}(s_i))^2, \quad 
                w^c_{k} \leftarrow \arg\min_w \sum_{i=1}^B (V_C(s_i;w) - V_C^{tar}(s_i))^2
            \end{equation*}

            \IF{$f_k==0$}
                \STATE
                \begin{equation*}
                    \theta_{k} \leftarrow \arg\max_\theta \Big\{ \frac{1}{B}\sum_{i=1}^B\frac{\pi_\theta(a_i|s_i)}{\pi_{\theta_k}(a_i|s_i)}A_R(s_i,a_i) + \phi(\hat{J}_C + \frac{1}{B}\sum_{i=1}^B\frac{\pi_\theta(a_i|s_i)}{\pi_{\theta_k}(a_i|s_i)}A_C(s_i,a_i) - d^k) \Big\}
                \end{equation*}
            \ELSIF{$f_k==1$}
                \STATE
                \begin{equation*}
                    \theta_{k} \leftarrow \arg\max_\theta \Big\{ -\frac{1}{B}\sum_{i=1}^B\frac{\pi_\theta(a_i|s_i)}{\pi_{\theta_k}(a_i|s_i)}A_C(s_i,a_i) + \phi( - \hat{J}_R - \frac{1}{B}\sum_{i=1}^B\frac{\pi_\theta(a_i|s_i)}{\pi_{\theta_k}(a_i|s_i)}A_R(s_i,a_i) + g^k) \Big\}   
                \end{equation*}
            \ELSIF{$f_k==2$}
                \STATE
                \begin{equation*}
                    \theta_{k} \leftarrow \arg\max_\theta \Big\{ -\frac{1}{B}\sum_{i=1}^B
                    KL(\pi_\theta(\cdot|s_i)||\pi_{old}(\cdot|s_i))
                    + \phi(\hat{J}_C + \frac{1}{B}\sum_{i=1}^B\frac{\pi_\theta(a_i|s_i)}{\pi_{\theta_k}(a_i|s_i)}A_C(s_i,a_i) - d^k) \Big\}                    
                \end{equation*}
            \ENDIF
        \ENDFOR        
    \ENDFOR 

\end{algorithmic}
\end{algorithm}

\begin{algorithm}[h]
   \caption{UpdateBudgets}
   \label{algo:update_budgets}
\begin{algorithmic} 
    \STATE \textbf{Input}: reward return queue $\mathcal{D}_R$, cost return queue $\mathcal{D}_C$.
    \IF{$k < n_e$}
        \STATE $f_k \leftarrow 0$, $d^k \leftarrow d^{k-1}$, $g^k \leftarrow g^{k-1}$ 
        \COMMENT{encourage exploration in the early stage}
        \STATE \textbf{Return}: $f_k, g^k, d^k$
    \ENDIF
    \IF{ queues $\mathcal{D}_R, \mathcal{D}_C$ both converges}
        \IF{$|mean(\mathcal{D}_C)-d_{des}|\leq \delta$}
            \STATE algorithm finishes.
        \ELSIF{$mean(\mathcal{D}_C) > d_{des}$}
            \STATE Update budget: $d^k \leftarrow d^{k-1} + k_p(d_{des} - mean(\mathcal{D}_C))$
            \STATE Switch to projection stage: $f_k \leftarrow 2$
        \ELSE
            \STATE Update budget: $d^k \leftarrow d^{k-1} + k(d_{des} - mean(\mathcal{D}_C))$
            \STATE Switch to max-reward stage: $f_k \leftarrow 0$
        \ENDIF
    \ELSE
        \IF{$f_{k-1} == 0$}
            \IF{reach maximum iterations of max-reward stage}
                \STATE Switch to min-cost stage: $f_k \leftarrow 1$
                \STATE Update reward budget: $g^k \leftarrow mean(\mathcal{D}_R)$, $d^k \leftarrow d^{k-1}$
            \ENDIF
        \ELSIF{$f_{k-1} == 1$}
            \IF{reach maximum iterations of min-cost stage}
                \STATE Switch to max-reward stage: $f_k \leftarrow 0$
                \STATE Update cost budget: $d^k \leftarrow mean(\mathcal{D}_C)$, $g^k \leftarrow g^{k-1}$
            \ENDIF
        \ENDIF
    \ENDIF
    \STATE \textbf{Return}: $f_k, g^k, d^k$
\end{algorithmic}
\end{algorithm}

\section{Experimental Details}

\subsection{Safety Gymnasium environments} \label{appendix:safety-gym-env}
Four environments used in experiments are illustrated in \Cref{fig:safety-gym-env}.

In the \texttt{Goal1} task, the agent aims to reach the green goal position while avoiding collision with blue hazards. If the agent enters the blue zone, it will receive a positive cost signal, which is proportional to the distance from the agent to the zone boundary. The reward signal consists of two parts: the first part indicates the change in the distance from the agent to the goal point. When the agent is closer to the goal than the previous time step, it receives a positive reward, vice versa a negative value. The second part of reward indicates whether the agent reaches the goal or not. If reaches, it will receive a large positive constant value.

In the \texttt{Circle1} task, the agent aims to follow the outermost circumference of the circle while staying within yellow lines. When the agent crosses the yellow boundary from the inside outward, it receives a constant cost signal. The reward function is designed as
\begin{equation}
    R_t = \frac{1}{1 + |r_{agent} - r_{circle}|} \cdot \frac{-uy+vx}{r_{agent}}
\end{equation}
where $u, v$ are the x-y axis velocity components of the agent, $x, y$ are the x-y axis coordinates of the agent, $r_{agent}$ is the Euclidean distance of the agent from the origin, $r_{circle}$ is the radius of the circle geometry.

The dimension of observation and action space in every environments are listed as follows:
\begin{itemize}
    \item PointGoal1: $\mathcal{S} \in \mathbb{R}^{60}, \mathcal{A}\in \mathbb{R}^2$
    \item CarGoal1: $\mathcal{S} \in \mathbb{R}^{72}, \mathcal{A}\in \mathbb{R}^2$
    \item PointCircle1: $\mathcal{S} \in \mathbb{R}^{28}, \mathcal{A}\in \mathbb{R}^2$
    \item CarCircle1: $\mathcal{S} \in \mathbb{R}^{40}, \mathcal{A}\in \mathbb{R}^2$
\end{itemize}

\begin{figure}[h]
    \centering
    \subfigure[PointGoal1]{
        \includegraphics[width=0.3\columnwidth]{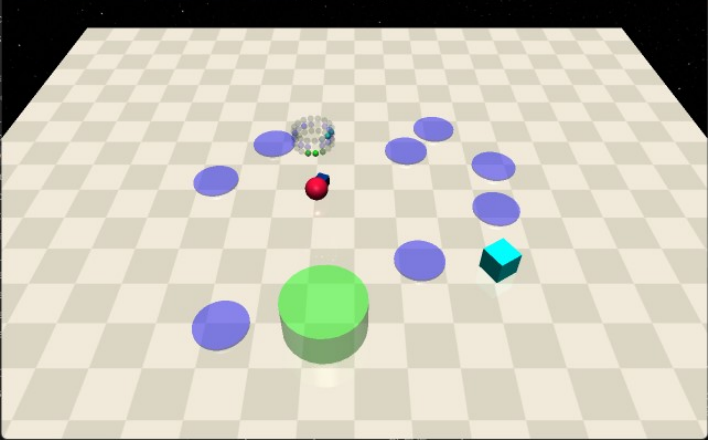}
    }
    \subfigure[PointCircle1]{
        \includegraphics[width=0.3\columnwidth]{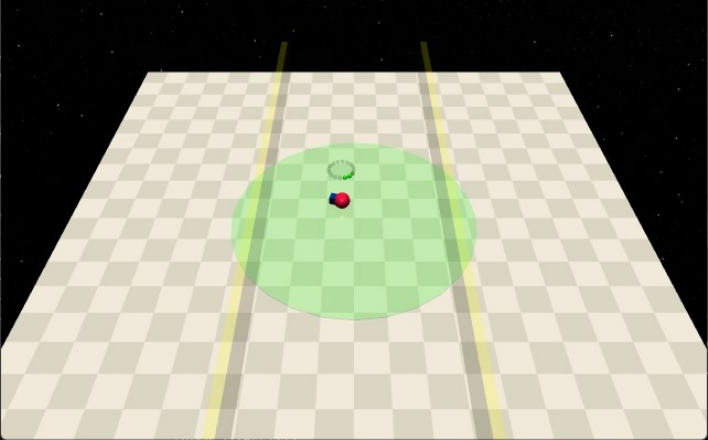}
    }
    
    \subfigure[CarGoal1]{
        \includegraphics[width=0.3\columnwidth]{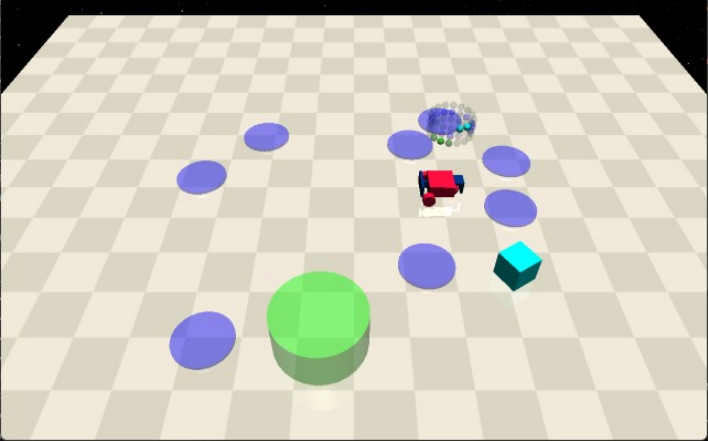}
    }
    \subfigure[CarCircle1]{
        \includegraphics[width=0.3\columnwidth]{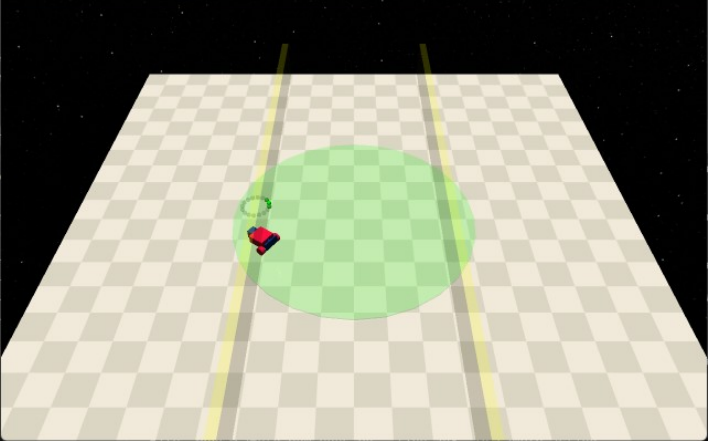}
    }
    \caption{An overview of single constraint environments.}
    \label{fig:safety-gym-env}
\end{figure}

\subsection{Multi-constraint quadruped environment}
In the quadruped locomotion task, the quadruped robot aims to track the desired velocity command while satisfying a set of constraints required by the task. The quadruped robot is shown in \Cref{fig:leggedrobot-env}. We use Isaac Gym \cite{makoviychuk2021isaac} to simulate robot dynamics and follow the method of \cite{leggedgym} to achieve parallel simulation. The observation and action space are provided in \Cref{table:leggedrobot-observation_space}, the reward functions are provided in \Cref{table:leggedrobot-rewards} and the cost functions are provided in \Cref{table:leggedrobot-costs}. 

\begin{figure}[h]
    \centering
    \includegraphics[width=0.3\columnwidth]{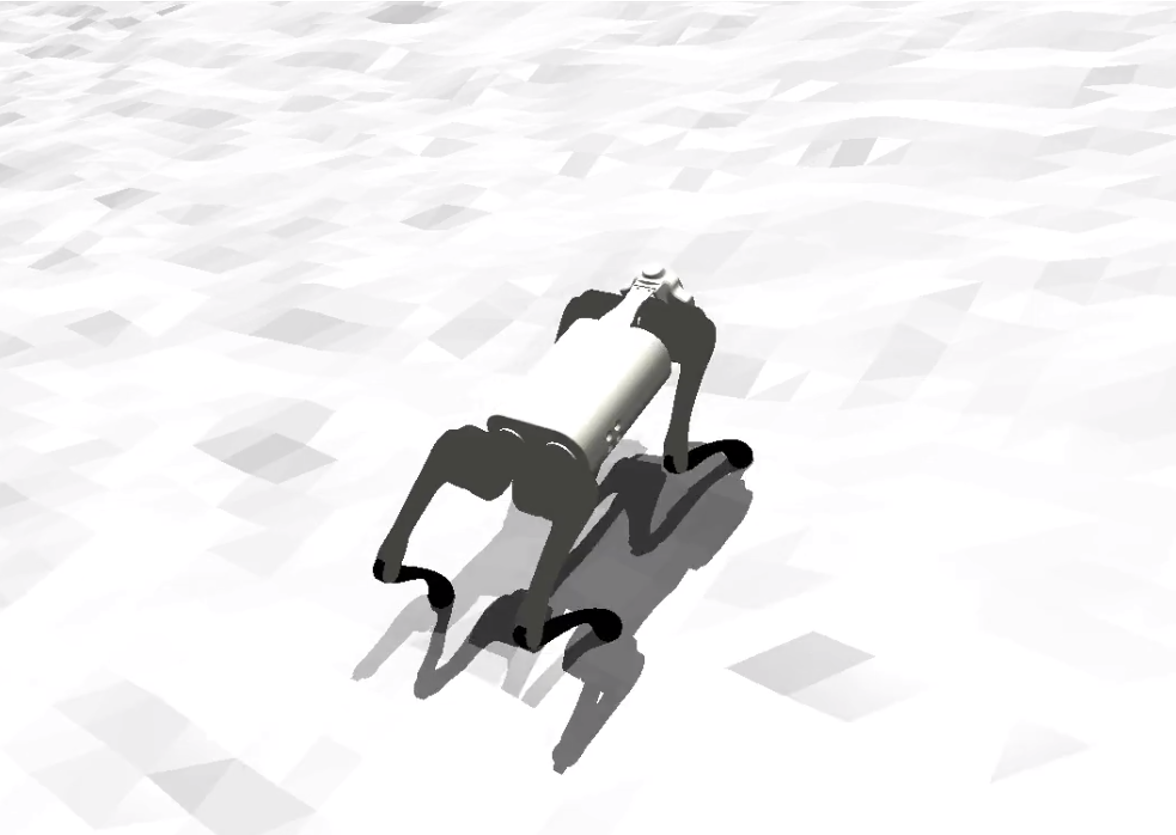}
    \caption{An illustration of the quadruped locomotion environment. }
    \label{fig:leggedrobot-env}
\end{figure}

\subsection{Algorithm Parameters} \label{section:parameters_algorithms}
Algorithm parameters used in Safety Gymnasium environments and the quadruped locomotion task are listed in \Cref{table:parameters-safetygym} and \Cref{table:parameters-leggedrobot}.

\subsection{Legged Robot Task Results} \label{section:leggedrobot-detailed-results}
\Cref{fig:leggedrobot-epcost}, \Cref{fig:leggedrobot-budgetcost} and \Cref{fig:leggedrobot-budgetreward} illustrate the detailed results in the quadruped locomotion task.

\begin{figure}
    \centering
    \includegraphics[width=\textwidth]{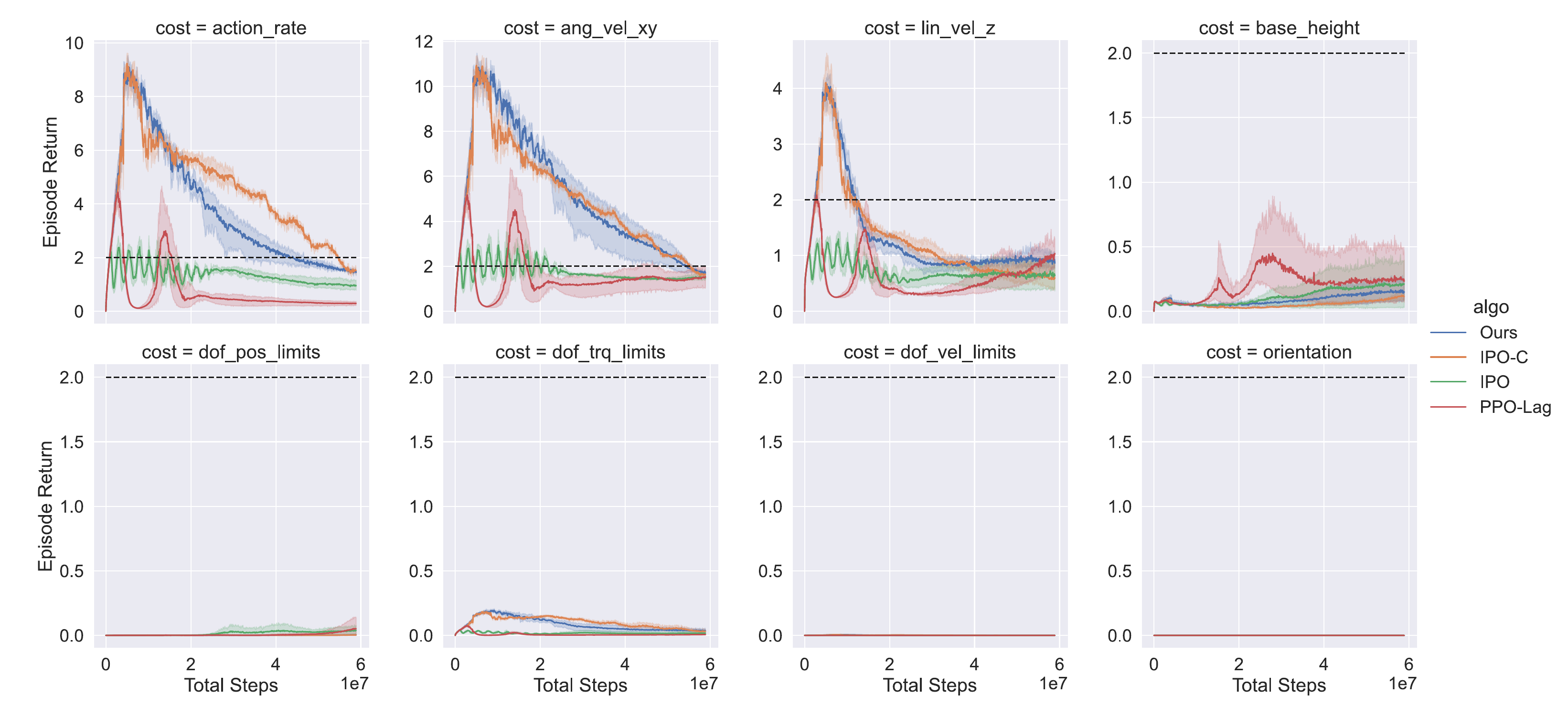}
    \caption{Episode cost return of different type of costs in the quadruped locomotion experiment. Results are the mean value of 5 training experiments using random seeds. The dashed line represents the desired cost limit = 2.}
    \label{fig:leggedrobot-epcost}
\end{figure}

\begin{figure}
    \centering
    \includegraphics[width=\textwidth]{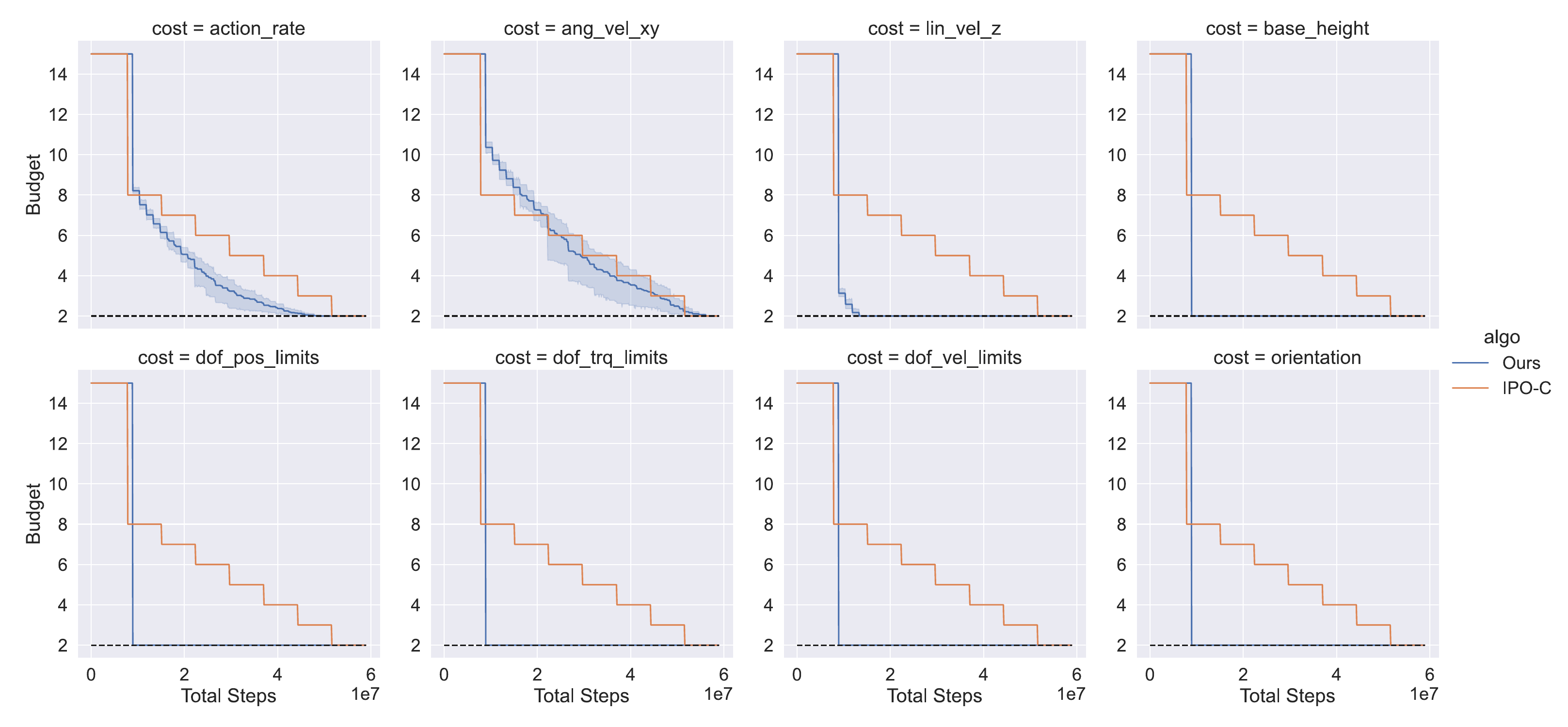}
    \caption{The change of cost budgets $d$ of ours and IPO with curriculum budgets during the training process in the quadruped locomotion task.}
    \label{fig:leggedrobot-budgetcost}
\end{figure}

\begin{figure}
    \centering
    \includegraphics[width=0.4\columnwidth]{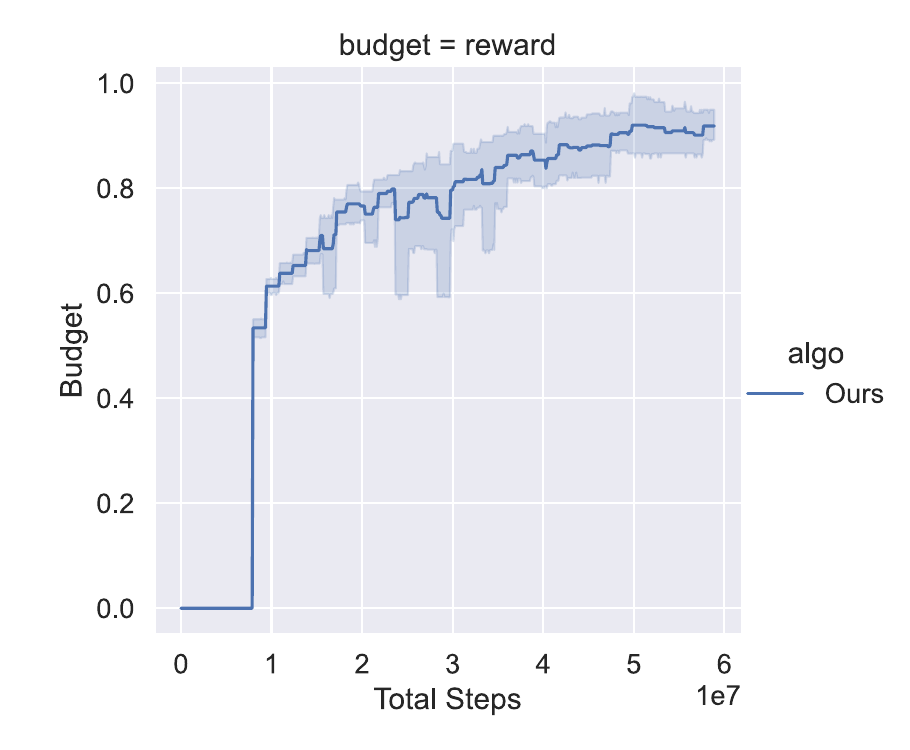}
    \caption{The change of reward budget $g$ during the training process in the quadruped locomotion task.}
    \label{fig:leggedrobot-budgetreward}
\end{figure}

\begin{table}[ht]
    \centering
    \caption{The observation and action space in the quadruped locomotion task.}
    \label{table:leggedrobot-observation_space}
    \begin{tabular}{ccc}
    \hline
    \multicolumn{3}{c}{\textbf{Observation Space}}                                       \\ \hline
    \multicolumn{1}{c|}{\textbf{Type}}           & \multicolumn{1}{c|}{\textbf{Dimension}} & \textbf{Scale} \\ \hline
    \multicolumn{1}{c|}{joint position}          & \multicolumn{1}{c|}{12}                 & 1     \\
    \multicolumn{1}{c|}{joint velocity}          & \multicolumn{1}{c|}{12}                 & 0.05   \\
    \multicolumn{1}{c|}{torso linear velocity }  & \multicolumn{1}{c|}{3}                  & 2     \\
    \multicolumn{1}{c|}{torso angular velocity } & \multicolumn{1}{c|}{3}                  & 0.25     \\
    \multicolumn{1}{c|}{gravity vector}          & \multicolumn{1}{c|}{3}                  & 0.3     \\
    \multicolumn{1}{c|}{actions}                 & \multicolumn{1}{c|}{12}                 & 1     \\
    \multicolumn{1}{c|}{last actions}            & \multicolumn{1}{c|}{12}                 & 1     \\
    \multicolumn{1}{c|}{feet contact states}     & \multicolumn{1}{c|}{4}                  & 1     \\
    \multicolumn{1}{c|}{gait phases}             & \multicolumn{1}{c|}{8}                  & 1     \\
    \multicolumn{1}{c|}{velocity commands}       & \multicolumn{1}{c|}{3}                  & 2     \\
    \hline
    \multicolumn{3}{c}{\textbf{Action Space}}                                                      \\ \hline
    \multicolumn{1}{c|}{\textbf{Type}}           & \multicolumn{1}{c|}{\textbf{Dimension}} & \textbf{Scale} \\ \hline
    \multicolumn{1}{c|}{desired joint position}  & \multicolumn{1}{c|}{12}                 & 0.25     \\ \hline
    \end{tabular}
\end{table}

\begin{table}[ht]
    \centering
    \caption{The reward functions used in the quadruped locomotion task. $\mathbf{v}_{xy}$ denotes the current torso velocity projected on the xy-plane. $\mathbf{w}_z$ is the current torso velocity projected on the z-axis. $\phi_i$ denotes the current gait phase of $i$th leg ranged from 0 to 1. $h_i, \mathbf{v}_i, \mathbf{f}_i$ denote the current foot height, foot velocity and ground reaction force of $i$th leg respectively. And the subscript $d$ represents the desired value.}
    \label{table:leggedrobot-rewards}
    \begin{tabular}{ccc}
    \hline
    \multicolumn{3}{c}{\textbf{Reward}}                                                         \\ \hline
    \multicolumn{1}{c|}{\textbf{Type}}             & \multicolumn{1}{c|}{\textbf{Formulation}}                         & \textbf{Weight}($w$)      \\ \hline
    \multicolumn{1}{c|}{linear velocity tracking}  & \multicolumn{1}{c|}{$\mathrm{exp}(-||\mathbf{v}_{xy}-\mathbf{v}_{xy,d}||^2/\sigma)$}          & 1        \\
    \multicolumn{1}{c|}{angular velocity tracking} & \multicolumn{1}{c|}{$\mathrm{exp}(-||\mathbf{w}_z-\mathbf{w}_{z,d}||^2/\sigma)$}               & 0.5      \\
    \multicolumn{1}{c|}{collision}                 & \multicolumn{1}{c|}{number of collision links}                    & -1       \\
    \multicolumn{1}{c|}{foot clearance}            & \multicolumn{1}{c|}{$\sum_{i\in\mathrm{foot}}(1-\phi_{i,d})\cdot[1 - \mathrm{exp}(-||h_i - h_{i,d}||^2/\sigma)]$}               & -1      \\
    \multicolumn{1}{c|}{gait velocity}             & \multicolumn{1}{c|}{$\sum_{i\in\mathrm{foot}}\phi_{i,d}\cdot[1 - \mathrm{exp}(-||\mathbf{v}_{i}||^2/\sigma)]$}               & -1      \\
    \multicolumn{1}{c|}{gait force}                & \multicolumn{1}{c|}{$\sum_{i\in\mathrm{foot}}(1-\phi_{i,d})\cdot[1 - \mathrm{exp}(-||\mathbf{f}_{i}||^2/\sigma)]$}               & -1      \\
    \hline
    \end{tabular}
\end{table}

\begin{table}[h]
\centering
\caption{The one-step cost functions used in the quadruped locomotion task. The subscript $l, r$ represent the lower and upper bounds respectively.}
\label{table:leggedrobot-costs}
\begin{tabular}{ccc}
\hline
\multicolumn{3}{c}{\textbf{Cost}}                                                                                  \\ \hline
\multicolumn{1}{c|}{\textbf{Type}}                & \multicolumn{1}{c|}{\textbf{Formulation}}  & \textbf{Boundary}($b$) \\ \hline
\multicolumn{1}{c|}{angular velocity on xy-plane} & \multicolumn{1}{c|}{$w$}             & $b_l=-1,b_r=1$                   \\
\multicolumn{1}{c|}{linear velocity in z-axis}    & \multicolumn{1}{c|}{$v_z$}                 & $b_l=-0.4,b_r=0.4$                   \\
\multicolumn{1}{c|}{orientation}                  & \multicolumn{1}{c|}{$||\mathbf{g}_{xy}||$} & $b_l=0,b_r=7$                   \\
\multicolumn{1}{c|}{torso height}                 & \multicolumn{1}{c|}{$h$}                     & $b_l=0.2,b_r=0.5$                   \\
\multicolumn{1}{c|}{joint position limit}         & \multicolumn{1}{c|}{$q$}                     & $b_l=-0.9\cdot limit, b_r=0.9\cdot limit$                   \\
\multicolumn{1}{c|}{joint velocity limit}         & \multicolumn{1}{c|}{$\dot{q}$}               & $b_l=-0.9\cdot limit, b_r=0.9\cdot limit$                   \\
\multicolumn{1}{c|}{joint torque limit}           & \multicolumn{1}{c|}{$\tau$}                  & $b_l=-0.7\cdot limit, b_r=0.7\cdot limit$                   \\
\multicolumn{1}{c|}{action change rate}           & \multicolumn{1}{c|}{$(a-a_{last})/dt$}       & $b_l=-10,b_r=10$                   \\ \hline
\end{tabular}
\end{table}

\begin{table*}[h]
\centering
\caption{Hyper-parameters for algorithms in Safety Gymnasium environment.}
\label{table:parameters-safetygym}
\begin{tabular}{lccccc}
\hline
Hyperparameter                        & Ours  & IPO   & CPO   & CRPO  & PPO-Lag \\ \hline
Number of Hidden layers               & 2     & 2     & 2     & 2     & 2       \\
Number of Hidden nodes                & 64    & 64    & 64    & 64    & 64      \\
Activation function                   & tanh  & tanh  & tanh  & tanh  & tanh    \\
Batch size                            & 20000 & 20000 & 20000 & 20000 & 20000   \\
SGD update iterations                 & 40    & 40    & 10    & 10    & 40      \\
Discount for reward                   & 0.99  & 0.99  & 0.99  & 0.99  & 0.99    \\
Discount for cost                     & 0.99  & 0.99  & 0.99  & 0.99  & 0.99    \\
GAE parameter for reward              & 0.95  & 0.95  & 0.95  & 0.95  & 0.95    \\
GAE parameter for cost                & 0.95  & 0.95  & 0.95  & 0.95  & 0.95    \\
Learning rate for policy              & 3e-4  & 3e-4  & N/A   & N/A   & 3e-4    \\
Learning rate for reward critic       & 3e-4  & 3e-4  & 1e-3  & 1e-3  & 3e-4    \\
Learning rate for cost critic         & 3e-4  & 3e-4  & 1e-3  & 1e-3  & 3e-4    \\
Learning rate for Lagrange multiplier & N/A   & N/A   & N/A   & N/A   & 0.035   \\
Initial value for Lagrange multiplier & N/A   & N/A   & N/A   & N/A   & 1e-3    \\
Upper bound for Lagrange multiplier   & N/A   & N/A   & N/A   & N/A   & 1000    \\
Optimizer for Lagrange multiplier     & N/A   & N/A   & N/A   & N/A   & Adam    \\
Trust region bound                    & 0.02  & 0.02  & 0.01  & 0.01  & 0.02    \\
Upper bound of penalty function       & 25    & 25    & N/A   & N/A   & N/A     \\
Clip ratio                            & 0.2   & 0.2   & N/A   & N/A   & 0.2     \\
Damping coeff.                        & N/A   & N/A   & 0.01  & 0.01  & N/A     \\
Backtracking coeff.                   & N/A   & N/A   & 0.8   & 0.8   & N/A     \\
Max backtracking iterations           & N/A   & N/A   & 15    & 15    & N/A     \\
Max conjugate gradient iterations     & N/A   & N/A   & 15    & 15    & N/A     \\ 
Max iterations of max-reward stage    & 10    & N/A   & N/A   & N/A   & N/A     \\
Max iterations of min-cost stage      & 5     & N/A   & N/A   & N/A   & N/A     \\
\hline
\end{tabular}
\end{table*}

\begin{table*}[h]
\centering
\caption{Hyper-parameters for algorithms in the quadruped locomotion task.}
\label{table:parameters-leggedrobot}
\begin{tabular}{lccc}
\hline
Hyperparameter                        & Ours  & IPO   & PPO-Lag \\ \hline
Number of Hidden layers               & 2     & 2     & 2       \\
Number of Hidden nodes                & 64    & 64    & 64      \\
Activation function                   & tanh  & tanh  & tanh    \\
Batch size                            & 16384 & 16384 & 16384   \\
SGD update iterations                 & 2     & 2     & 2       \\
Discount for reward                   & 0.99  & 0.99  & 0.99    \\
Discount for cost                     & 0.99  & 0.99  & 0.99    \\
GAE parameter for reward              & 0.95  & 0.95  & 0.95    \\
GAE parameter for cost                & 0.95  & 0.95  & 0.95    \\
Learning rate for policy              & 3e-3  & 3e-3  & 3e-3    \\
Learning rate for reward critic       & 3e-3  & 3e-3  & 3e-3    \\
Learning rate for cost critic         & 3e-3  & 3e-3  & 3e-3    \\
Learning rate for Lagrange multiplier & N/A   & N/A   & 0.035   \\
Initial value for Lagrange multiplier & N/A   & N/A   & 1e-3    \\
Upper bound for Lagrange multiplier   & N/A   & N/A   & 1000    \\
Optimizer for Lagrange multiplier     & N/A   & N/A   & Adam    \\
Trust region bound                    & 0.02  & 0.02  & 0.02    \\
Upper bound of penalty function       & 1     & 1     & N/A     \\
Clip ratio                            & 0.2   & 0.2   & 0.2     \\
Max iterations of max-reward stage    & 20    & N/A   & N/A     \\
Max iterations of min-cost stage      & 10    & N/A   & N/A     \\
\hline
\end{tabular}
\end{table*}


\end{document}